\documentclass{article}

 \usepackage[margin=1in]{geometry}
\usepackage{natbib}
\usepackage[utf8]{inputenc} 
\usepackage[T1]{fontenc}    
\usepackage{lmodern}
\usepackage{hyperref}       
\usepackage{url}            
\usepackage{booktabs}       
\usepackage{amsfonts}       
\usepackage{nicefrac}       
\usepackage{microtype}      
\hypersetup{
  colorlinks=true,
  linkcolor=blue,
  citecolor=blue,
  urlcolor=magenta
}
\usepackage{xcolor}         
\usepackage{amsmath,amsthm,amssymb}
\usepackage{mathtools}

\usepackage{cleveref}

\usepackage{algorithm}
\usepackage{algorithmicx}
\usepackage{algpseudocode}

\usepackage{thmtools}
\usepackage{thm-restate}

\makeatletter
\AtBeginDocument{%
    \patchcmd{\cref@thmnoarg}{\refstepcounter}{\Hy@theorem@refstepcounter}
        {}{\PackageError{neurips_macros}{Could not patch theorem hyperlink placement}{}}%
}
\makeatother

\makeatletter
\@ifundefined{newcounteralias}{}{%
    \renewcommand\thmt@autorefsetup{%
        \@xa\def\csname\thmt@envname autorefname\@xa\endcsname\@xa{\thmt@thmname}%
    }%
}
\makeatother

\DeclareMathAlphabet{\mathmybb}{U}{bbold}{m}{n}
\newcommand{\eps}{\varepsilon}
\newcommand{\ind}{\mathbb{I}}

\DeclareMathOperator*{\e}{\mathbb{E}}
\DeclareMathOperator*{\p}{\mathbb{P}}

\DeclareMathOperator*{\argmin}{\textit{argmin}}
\DeclareMathOperator*{\argmax}{\textit{argmax}}

\DeclareMathOperator{\supp}{\mathrm{supp}}

\theoremstyle{plain}
\newtheorem{theorem}{Theorem}[section]

\newtheorem{lemma}[theorem]{Lemma}
\newtheorem{corollary}[theorem]{Corollary}

\theoremstyle{definition}

\theoremstyle{remark}

\makeatletter
\newcommand{\alphacmd@factory}[1]{}
\newcounter{alphacmdcounter}
\newcommand{\GenerateAlphabetCmds}[2]{%
    \renewcommand{\alphacmd@factory}[1]{%
        \expandafter\providecommand\csname #1##1\endcsname{{#2{##1}}}%
    }
    \setcounter{alphacmdcounter}{0}
    \loop
        \stepcounter{alphacmdcounter}
        \edef\alphacmd@ID{\@Alph\c@alphacmdcounter}
        \expandafter\alphacmd@factory\alphacmd@ID
    \ifnum\thealphacmdcounter<26
    \repeat
}
\newcommand{\GenerateAlphabetCmdsLower}[2]{%
    \renewcommand{\alphacmd@factory}[1]{%
        \expandafter\providecommand\csname #1##1\endcsname{{#2{##1}}}%
    }
    \setcounter{alphacmdcounter}{0}
    \loop
        \stepcounter{alphacmdcounter}
        \edef\alphacmd@ID{\@alph\c@alphacmdcounter}
        \expandafter\alphacmd@factory\alphacmd@ID
    \ifnum\thealphacmdcounter<26
    \repeat
}
\makeatother

\GenerateAlphabetCmds{c}{\mathcal}
\GenerateAlphabetCmdsLower{c}{\mathcal}

\GenerateAlphabetCmds{b}{\mathbb}
\GenerateAlphabetCmdsLower{b}{\mathbb}

\GenerateAlphabetCmds{r}{\mathbf}
\GenerateAlphabetCmdsLower{r}{\mathbf}

\def\gap{ {\mathrm{gap}} }
\def\iderr{{\mathrm{IdErr}}}
\def\iderrmin{{\mathrm{IdP}}}
\def\generr{{\mathrm{GenErr}}}
\def\loss{{\mathrm{Err}}}
\def\optarg{{\mathrm{Opt}}}

\def\generravg{{\mathrm{GenErr_{avg}}}}

\title{Agnostic Language Identification and Generation}

\author{
  Mikael M{\o}ller H{\o}gsgaard\thanks{University of Oxford.
  Email: \texttt{mikael.mollerhogsgaard@stats.ox.ac.uk}.}
  \and
  Chirag Pabbaraju\thanks{Stanford University.
  Email: \texttt{cpabbara@cs.stanford.edu}.}}
\date{\today}

\begin{document}

\maketitle

\begin{abstract}
  Recent works on language identification and generation have established tight statistical rates at which these tasks can be achieved. These works typically operate under a strong realizability assumption: that the input data is drawn from an unknown distribution necessarily supported on some language in a given collection. In this work, we relax this assumption of realizability entirely, and impose no restrictions on the distribution of the input data. We propose objectives to study both language identification and generation in this more general ``agnostic'' setup. Across both problems, we obtain novel interesting characterizations and nearly tight rates.
\end{abstract}

\section{Introduction}
\label{sec:introduction}

Learning a language from a finite amount of data, and in particular from finitely many positive examples from the language, is a fundamental problem, both for humans as well as computers. Two natural language learning problems are those of \textit{language identification} and \textit{language generation}. In language identification, the objective is to output an exact representation of the language that is presumed to be generating the examples being seen, whereas the objective in language generation is to simply be able to generate new valid examples from the unknown language, despite potentially not learning an exact representation of it. In this regard, the latter task is an easier task than the former.

An extremely bare-bones framing of this problem, devoid of any structural assumptions on language, is as follows. A language may simply be thought of as an abstract set $L$ of valid sentences or strings. If the alphabet that defines these sentences is finite (e.g., the finite English alphabet \{a,b,...,z\}), and sentences themselves are finite, then a language may further be thought of as a countable subset of a countable universe $U$ of all possible strings. Stated thus, the language identification task may be framed as follows: Given a finite dataset $S \subseteq L$ of positive examples from $L$, output (possibly the representation of) a set $L'$ such that $L'=L$. On the other hand, the language generation task becomes the following: Given a finite dataset $S \subseteq L$ of positive examples from $L$, output a new unseen example $z$ from $L$. In order to make these problems tractable, one typically assumes \textit{realizability}: namely, that the unknown language $L$ is one from a known collection $\cC=\{L_1,L_2,\dots\}$ of languages. Furthermore, the way that the dataset $S \subseteq L$ is obtained is up to further modeling: it may be generated as an online sequence by an adversary, or it may be generated as an i.i.d. sequence from a distribution supported on $L$.

The setup where the data may be generated in an online, adversarial fashion has been extensively studied in the literature. Several classical results have been established for language identification \citep{gold1967language,angluin1980inductive}, and more recently, several results have been established for language generation \citep{kleinberg2024language,li25generation,charikar25exploring}. As it turns out, the online language identification problem is in general quite hard, and intractable for many natural formal language collections even with the realizability assumption. In stark contrast, the online language generation problem is tractable for every countable language collection. Coming to the statistical setting, where the examples are generated i.i.d. from a distribution over the unknown target language, recent work by \citet*{kalavasis2025limits} studies algorithms that minimize the probability of identifying an incorrect language, or generating a string outside the underlying target language, as a function of the number $n$ of i.i.d. examples seen from the language. The results of \citet{kalavasis2025limits} characterize precise rates at which these tasks may be achieved. Importantly, all these results operate under the realizability assumption---that the data is generated entirely from some language from the reference collection.

Since the online identification problem is largely intractable even with this realizability assumption, whereas the generation problem is much more tractable, recent works \citep{raman25noisy,mehrotra2025language,bai2026language} have also studied the online generation problem beyond the realizable setting, where the input data stream may be corrupted to have noisy examples outside the target language (but without any signal as to which examples are noisy). This setting, which may be termed as the ``agnostic'' setting, is arguably more representative of real-world datasets. The results in these works obtain precise conditions on the nature of noisy examples in the input stream that allow for successful generation in the online setting. However, the study of a similar agnostic setup in the statistical setting where the data is drawn from a distribution has so far been lacking in the literature.

In this work, we aim to bridge this gap, by considering the agnostic setting for language identification and generation, where the data is drawn from an arbitrary distribution over the universe, with support not necessarily equal to any language from the reference collection. We propose reasonable objectives to study these problems in this more general agnostic setting, and derive statistical rates at which the objectives may be achieved.

\subsection{Agnostic identification}
\label{sec:intro-agnostic-identification}

The precise statistical objective studied by \citet{kalavasis2025limits} for language identification in the realizable setting is the following. Let $\cD$ be an unknown distribution over the universe $U$, and let $\supp(\cD)$ denote its support (i.e., elements of $U$ that have positive mass under $\cD$). Given as input an i.i.d. dataset $\rS \sim \cD^n$, where $\cD$ satisfies $\supp(\cD)=L^\star$ for some $L^\star \in \cC$, the task of an identification algorithm $\cA$, possibly using randomness $\rr$, is to identify an index $\cA(\rS, \rr)$ so as to minimize:
\begin{align}
    \label{eqn:realizable-identification-benchmark}
    \e_{\rS \sim \cD^n\hspace{-0.2em},\hspace{0.1em} \rr}\left[\ind\{L_{\cA(\rS,\rr )} \neq L^\star\}\right].
\end{align}
\citet[Proposition 3.10]{kalavasis2025limits} shows that for any collection that can be identified in the (realizable) online setting, there exists an algorithm for which the quantity above goes down exponentially with the sample size $n$.

In the agnostic setting, given that the support of the underlying distribution $\cD$ may not exactly coincide with any language in the reference collection, a natural broader objective is to instead identify a language in the collection that is \textit{most representative} of the distribution $\cD$, in the sense that it is most likely to contain strings drawn from $\cD$. Towards this, let $\loss_{\cD}(L):=\p_{\rx\sim \cD}[\rx\not\in L]$, and consider the following objective:
\begin{align}
    &\iderr(\cA, \cD, \cC, n) := \e_{\rS \sim \cD^n\hspace{-0.2em},\hspace{0.1em} \rr}\big[\loss_{\cD}\!\left(L_{\cA(\rS, \rr)}\right)-\inf_{L\in \cC} \loss_{\cD}(L)\big]. \label{eqn:agnostic-identification-benchmark}
\end{align}
We call this objective \textit{weak identification}, where we deliberately use the qualifier ``weak'' to distinguish it from exact Gold-style identification of the unknown target language. Rather, it is a one-sided distributional benchmark: it rewards a language $L$ that covers as much mass of $\cD$ as possible, but it does not penalize strings in $L\setminus\supp(\cD)$, which are statistically invisible given only positive i.i.d.\ samples from $\cD$.

To get a sense of this objective, and how it compares to the previous realizable identification notion \eqref{eqn:realizable-identification-benchmark}, let us consider the special realizable case, where $\supp(\cD)=L^\star$ for some $L^\star \in \cC$. Then $\inf_{L\in \cC} \loss_{\cD}(L)=0$, so \eqref{eqn:agnostic-identification-benchmark} becomes $\e_{\rS \sim \cD^n\hspace{-0.2em},\hspace{0.1em} \rr}[\loss_{\cD}(L_{\cA(\rS, \rr)})]$, which is always at most the exact-identification error in \eqref{eqn:realizable-identification-benchmark}. Thus, the benchmark is weaker than the usual realizable one in this setting. The weakening is nevertheless still useful because exact identification is generally too stringent: as established by \citet[Theorem 3.1]{kalavasis2025limits}, any collection that admits a vanishing rate for the objective in \eqref{eqn:realizable-identification-benchmark} must also be identifiable in the online setting, and such collections must satisfy a very restrictive condition \citep{gold1967language,angluin1980inductive}. This makes exact identification intractable for many natural language collections. On the other hand, using the weaker benchmark above allows the learner (for example) to identify a superset of the support of the distribution, which is a more tractable yet meaningful task.

The above discussion also implies, that \textit{even} if we restrict ourselves to special cases of distributions $\cD$ for which some language $L^\star \in \cC$ satisfies $\loss_{\cD}(L^\star)=\inf_{L\in \cC} \loss_{\cD}(L)$, it would still be hopeless to obtain a vanishing rate on an objective of the form \eqref{eqn:realizable-identification-benchmark} in general. This further motivates considering the weaker objective in \eqref{eqn:agnostic-identification-benchmark}, which allows us to gracefully handle both, distributions $\cD$ satisfying $\supp(\cD) \neq L^\star$ for any $L^\star \in \cC$, as well as collections $\cC$ that are not identifiable in the online setting. Indeed, the first result shows that we can make the objective in \eqref{eqn:agnostic-identification-benchmark} go down at an (almost) \textit{exponential} rate, for any countable collection $\cC$ and distribution $\cD$ satisfying that there exists some $L^\star \in \cC$ which attains $\inf_{L \in \cC} \loss_{\cD}(L)$; in other words, if we define $ \optarg_{\cD,\cC}:=\argmin_{L\in \cC}\loss_{\cD}(L)$, then $\optarg_{\cD,\cC} \neq \emptyset$. Note again that the result guarantees such a rate even when the collection may not be identifiable in the online setting, and the support of the distribution is not equal to any language in the collection.

\begin{theorem}[Informal, see \Cref{thm:agnotic-identification-exponential-rate,thm:agnotic-identification-exponential-rate-optimal}]
    \label{thm:informal-agnostic-identification-ub}
    Let $\cC$ be any countable collection and $\cD$ be any distribution. If $\optarg_{\cD,\cC} \neq \emptyset$, then for any function $g(n)=o(n)$, there exists an identification algorithm $\cA$ that satisfies
    \[
        \iderr(\cA, \cD, \cC, n) \lesssim e^{-g(n)}.
    \]
    Moreover, under this assumption, the best that any algorithm $\cA$ can do is $\iderr(\cA, \cD, \cC, n) \gtrsim e^{-n}.$
\end{theorem}
The informal statement above is made precise in \Cref{sec:exponential-identification-rate,sec:exponential-identification-rate-best}. In fact, the algorithm has the stronger property that the language it identifies \textit{attains} $\inf_{L \in \cC}\loss_\cD(L)$ with (almost) exponentially high probability.

The identification algorithm requires that the infimum $\inf_{L \in \cC}\loss_\cD(L)$ be attained by some language in the collection. Remarkably, the next result shows that this assumption is necessary in a very strong sense: if the infimum is not required to be attained within the collection, there exist distributions that necessitate arbitrarily slow rates!

\begin{theorem}[Informal, see \Cref{thm:arbitrary-slow-rate-identification}]
    \label{thm:informal-agnostic-identification-lb}
    There exists a countable collection $\cC$, such that for any identification algorithm $\cA$ and rate function $R(n)=o(1)$, there exists a distribution $\cD$ for which $\optarg_{\cD,\cC} = \emptyset$,
    and furthermore,  for infinitely many $n$,
    \[
    \iderr(\cA, \cD, \cC, n) \gtrsim R(n).
    \]
\end{theorem}

\Cref{thm:informal-agnostic-identification-ub,thm:informal-agnostic-identification-lb} together establish that agnostic language identification, at least in the sense of the weak objective in \eqref{eqn:agnostic-identification-benchmark}, is \textit{tightly characterized} by a rather surprising criterion: that the quantity $\loss_\cD(L)$ across languages in the collection  be infimized within the collection. If this criterion holds, then one can minimize \eqref{eqn:agnostic-identification-benchmark} at an optimal, exponentially fast rate; if it doesn't hold, then a priori one can only guarantee arbitrarily slow rates. The qualitative takeaway is as follows: if all that the language collection has are languages that get arbitrarily close to the infimum, but never quite attain it, then no matter how slowly an algorithm chases the infimum, it can't afford to chase it any faster.
As mentioned above, the benchmark for weak identification is one-sided, where identification of any language attaining the infimum is sufficient.
One might ask, however, if we can ensure some kind of \textit{minimality} of the identified language. As a pathological example, suppose $\cC$ contains $U$ as a language; in this case, we can trivially achieve the weak identification benchmark by always outputting $U$. However, in this case, there might be a much smaller language in the collection that also attains the infimum. To this end, we also study a stronger notion of agnostic identification, which we call \textit{minimal identification}. We define the minimal identification error for a distribution $\cD$ and language collection $\cC$ as follows:
\begin{align}
    &\iderrmin(\cA, \cD, \cC, n) := \p_{\rS \sim \cD^n\hspace{-0.2em},\hspace{0.1em} \rr}\big[ L_{\cA(\rS, \rr)}\not\in \optarg_{\cD,\cC} \text{ or } \exists L\in  \optarg_{\cD,\cC} \text{ s.t. } L\subsetneq L_{\cA(\rS, \rr)}\big]. \label{eqn:minimal-agnostic-identification-benchmark}
\end{align}
Note that if $\optarg_{\cD,\cC}$ is empty, the above probability is 1, so all algorithms perform equally.\footnote{Alternatively, we could have introduced an $\eps$ in the definition and only asked for $\eps$-optimality; we choose not to pursue this for simplicity.} But in this case, by \Cref{thm:informal-agnostic-identification-lb}, even the weak identification benchmark must suffer arbitrarily slow rates. So, for the minimal identification benchmark, we focus on achievable performance only within the class of distributions where the infimum is attained:
\begin{align}\label{eq:inf_attained_distributions}
  \cP_{Opt}(\cC)= \{ \cD \mid \optarg_{\cD,\cC} \neq \emptyset \} .
\end{align}
Additionally, if $\supp(\cD)=L$ for some $L\in \cC$, the minimal identification benchmark is \textit{equivalent} to the realizable identification benchmark in \cref{eqn:realizable-identification-benchmark}, thus making it a different natural generalization.

As it turns out, minimal identification has a deeper connection to the realizable benchmark in \eqref{eqn:realizable-identification-benchmark}, specifically through the notion of \textit{tell-tale sets}. Recall that by the results of \citet{kalavasis2025limits}, a vanishing rate for \eqref{eqn:realizable-identification-benchmark} is possible only if the collection is identifiable in the online setting. Indeed, identifiability in the online setting has an elegant characterization in the form of the existence of finite tell-tale sets for every language in the collection \citep{angluin1980inductive}.
Concretely, for $L\in\cC$, a finite set $T_L\subseteq L$ is called a \textit{tell-tale set} if
\[
    T_L\nsubseteq L'
    \qquad\text{for every }L'\in\cC\text{ satisfying }L'\subsetneq L.
\]
Let $\cT_L$ denote the collection of such tell-tale sets for $L$.
Now consider the following sub-class of distributions in $\cP_{Opt}$, which satisfy the added condition that at least one optimal language have a tell-tale in the support:
\begin{align}\label{eq:minimal-distributions}
    \cP(\cC) = \{\cD \in \cP_{Opt}(\cC) \mid \exists L \in \optarg_{\cD,\cC}, \exists T_L\in \cT_L, T_L\subseteq \supp(\cD)\}.
\end{align}
We show that there exists an algorithm, which for any distribution in $\cP(\cC)$, achieves minimal identification error going to $0$ at an almost exponential rate. Furthermore, we give evidence that the added condition in $\cP(\cC)$ is necessary for minimal identification: that there exists a universe $U$ and language collection $\cC$, where all languages have finite tell-tale sets, but any algorithm must fail on some $\cD \in \cP_{Opt}(\cC)\setminus \cP(\cC)$, suffering an error lower bounded by a constant strictly smaller than $1/2$.
\begin{theorem}[Informal, see \Cref{thm:agnostic-minimal-identification-exponential-rate,thm:necessity-supported-tell-tales}]
    \label{thm:informal-agnostic-minimal-identification-ub}
    Let $\cC$ be any countable collection of languages from a countable universe $U$, and $\cD\in \cP(\cC)$. For any function $g(n)=o(n)$, there exists an identification algorithm $\cA$ that satisfies
    \[
        \iderrmin(\cA, \cD, \cC, n) \lesssim e^{-g(n)}.
    \]
    Moreover, there exists a countable universe $U$ and a language collection $\cC$ with all languages having finite tell-tale sets, such that for any deterministic algorithm $\cA$ achieving $\iderrmin(\cA, \cD, \cC, n)\to 0$ for all $\cD\in \cP(\cC)$, there must exist a distribution $\cD\in \cP_{Opt}(\cC)\setminus \cP(\cC)$ for which $\iderrmin(\cA, \cD, \cC, n)\gtrsim 1/2-\varepsilon$ for infinitely many $n$.
\end{theorem}

With these identification benchmarks established, we make one final comment regarding the objectives in \eqref{eqn:agnostic-identification-benchmark} and \eqref{eqn:minimal-agnostic-identification-benchmark}. Identifying a language $L^\star$ does not automatically imply that it can be used to generate from the distribution $\cD$. This is in contrast with the realizable setting, where identification provides a direct route to generation: once an algorithm recovers $L^\star=\supp(\cD)$, generating a valid string reduces to choosing an element from $L^\star\setminus \rS$. Such a connection no longer has a clear analogue once realizability is removed, regardless of the agnostic identification benchmark used. The issue is the missing information about $U\setminus\supp(\cD)$. For instance, if every language $L\in\cC$ contains infinitely many strings outside $\supp(\cD)$, the natural candidate pool $L\setminus \rS$ of any identified language still contains infinitely many invalid strings. We formalize this basic obstruction to the ``identification implies generation'' principle in \Cref{thm:agnostic-generation-lower-bound}.

\subsection{Agnostic generation}
\label{sec:intro-agnostic-generation}

We now turn towards studying language generation in the agnostic setting. We start again by revisiting the generation objective studied by \citet{kalavasis2025limits} in the realizable setting: Given as input an i.i.d. dataset $\rS \sim \cD^n$, where $\cD$ satisfies $\supp(\cD)=L^\star$ for some $L^\star \in \cC$, the objective of a generation algorithm $\cA$, possibly using randomness $\rr$, is to output a string $\cA(\rS,\rr)$ so as to minimize:
\begin{align}
    \label{eqn:realizable-generation-benchmark}
    \e_{\rS \sim \cD^n\hspace{-0.2em},\hspace{0.1em} \rr}\left[\ind\{\cA(\rS,\rr) \notin L^\star \setminus \rS\}\right].
\end{align}
\citet[Theorem 3.2]{kalavasis2025limits} show that for \textit{every} countable language collection, there exists a generation algorithm that achieves an exponential rate on this objective. Note that it is necessary to exclude $\rS$ from $L^\star$ in the objective; otherwise the algorithm can trivially output an example from $\rS$ itself.

In the agnostic setting, where the support of the distribution may be arbitrary, a natural thing to do is to simply replace $L^\star$ with $\supp(\cD)$. This gives rise to the following objective for agnostic generation:
\begin{align}
    &\generr(\cA, \cD, n) := \e_{\rS \sim \cD^n\hspace{-0.2em},\hspace{0.1em} \rr}\left[\ind\{\cA(\rS,\rr) \notin \supp(\cD) \setminus \rS\}\right]. \label{eqn:agnostic-generation-benchmark}
\end{align}
When $\cD$ is allowed to be arbitrary, the right-hand side above is entirely independent of the reference collection.
If an algorithm has no reference to base the support of the unknown distribution, how is it supposed to generate new strings from its support? Indeed, the first result formally shows that this objective is intractable to achieve in general.

\begin{theorem}[Informal, see \Cref{thm:agnostic-generation-lower-bound}]
    \label{thm:informal-agnostic-generation-lb}
    Fix any $\eps \in (0,1)$. For any generation algorithm $ \cA $, there exists a distribution $ \cD $ such that for infinitely many $n$
    \[
    \generr(\cA, \cD, n) \geq 1-\eps.
    \]
\end{theorem}

The negative result above either necessitates imposing structural assumptions tying the distribution to the collection, or introducing a comparative term with reference to the collection into the objective (like how we had the term $\inf_{L \in \cC}\loss_\cD(L)$ to compare to for identification).
However, while there was a natural way to measure the ``identification error'' of a language $L$ in the collection as $\loss_{\cD}( L)$, it is not clear what the ``generation error'' of a language $L$ ought to be, and how one might compare it to the term in \eqref{eqn:agnostic-generation-benchmark}. Instead of outputting a single string, one might consider $\cA(\rS, \rr)$ to be a \textit{distribution} output by $\cA$; in this case, it is tempting to consider the reference collection $\cC$ to comprise of distributions, instead of languages. Unfortunately, the requirement to not output anything from the input sample introduces idiosyncrasies in this model. For example, the benchmark distribution to compare with no longer gets fixed by the input distribution alone, but may depend on the input sample or its size. This is discussed further in \Cref{sec:appendix-agnostic-generation-other-benchmarks}.

Thus, we turn towards studying minimal structural assumptions we can impose on the distribution $\cD$ which allow us to obtain meaningful rates for generation. As a reference point, recall that stipulating $\supp(\cD) = L^\star$ for some $L^\star \in \cC$, which brings us back to the realizable setting, allows obtaining exponential rates for \eqref{eqn:agnostic-generation-benchmark}. Perhaps surprisingly, we find that for finite collections, it suffices to assume a weaker condition, namely $\supp(\cD) \supseteq L$ for some $L \in \cC$, to achieve exponential rates. Note that this setting allows infinitely many strings outside any language from the collection to be seen in the input, and hence considerably extends beyond the realizable setting.

\begin{theorem}[Informal, see \Cref{thm:agnostic-generation-exponential-rate,thm:generation-exponential-rate-lower-bound}]
    Let $ \cC$ be any finite collection of languages and $\cD$ be any distribution. If  $\supp(\cD) \supseteq L  $ for some $L \in \cC$,
    then there exists a generation algorithm $ \cA$ that satisfies
    \[
    \generr(\cA, \cD, n) \lesssim e^{-n}.
    \]
    Moreover,
    under this assumption, the best that any algorithm can do is $\generr(\cA, \cD, n) \gtrsim e^{-n}$.
\end{theorem}

We remark that the generation algorithm works more generally for any countable collection, under the assumption that the distribution satisfies an intuitive condition stated in \Cref{sec:agnostic-generation-exponential-rate}. In fact, the algorithm is qualitatively different from previous algorithms, several of which crucially depend on the input being consistent with some language in the collection. We furthermore shows that if this condition is not met, then one can not guarantee a vanishing rate. Specifically, we show that for any deterministic algorithm, there exists a distribution not satisfying the above condition, for which the algorithm fails to achieve a vanishing rate.

\section{Agnostic language identification}
\label{sec:agnostic-identification}

In this section, we describe all the results about agnostic language identification. We start with the results about weak identification.
Recall the objective in weak identification: Given as input a sample $\rS$ of $n$ examples drawn i.i.d. from an unknown distribution $\cD$, an algorithm $\cA$, using randomness $\rr$, predicts an index $\cA(\rS, \rr)$ in reference to a countable collection $\cC=\{L_1,L_2,\dots\}$, so as to minimize the objective:
\begin{align*}
    &\iderr(\cA, \cD, \cC, n) := \e_{\rS \sim \cD^n\hspace{-0.2em},\hspace{0.1em} \rr}\big[\loss_{\cD}\!\left(L_{\cA(\rS, \rr)}\right)-\inf_{L\in \cC} \loss_{\cD}(L)\big].
\end{align*}
For notational convenience, we will denote the quantity $\loss_{\cD}(L')-\inf_{L\in \cC} \loss_{\cD}(L)$ to be the \textit{excess error} of language $L'$. Furthermore, given a sample $\rS \sim \cD^n$, we will use $\rx \sim \rS$ to denote drawing $\rx$ from the uniform distribution over $\rS$. In this case, the quantity $\loss_{\rS}(L)$ denotes the empirical error of language $L$.

\subsection{Almost exponential rate when \texorpdfstring{$\inf$}{inf} attained}
\label{sec:exponential-identification-rate}

To begin, let us assume that the collection $\cC$ and the underlying distribution $\cD$ satisfy that $\exists L^\star \in \cC$ satisfying $\loss_{\cD}(L^\star)=\inf_{L \in \cC}\loss_{\cD}(L)$. In this case, we will give an algorithm that achieves an (almost) exponentially decaying excess error rate.\footnote{Since the algorithm only consider the first $ f(n) $ languages in the collection, and the empirical-error involves only checking the first $ n $ samples, being in a language, the algorithm can be implemented in time $ O(f(n)n) $ given query access to elements being in a language.}

Arguably, the first natural algorithm to consider here is \textit{Empirical Risk Minimization}. Namely, consider an enumeration of the collection as $\cC=\{L_1,L_2,\dots\}$. Then, consider the algorithm $\cA$ that returns $L_{\cA(\rS)}$, where $\cA(\rS) \in \arg\min_{i \in \{1,2,\dots, f(n)\}} \loss_{\rS}(L_i)$, for any function $f(n)$ that increases to infinity. Namely, $\cA$ returns a language that minimizes the empirical error incurred on the sample from within a finite window of the collection, whose size $f(n)$ increases with the sample size $n$. Unfortunately, standard analyses only guarantee a polynomial rate for this algorithm.

To see this, observe that Hoeffding's inequality, together with the union bound, guarantees that the empirical errors $\loss_{\rS}(L_i)$ computed by $\cA$ estimate $\loss_{\cD}(L_i)$ to within a precision of $\eps$ simultaneously for every $L_i$ in the window $\{L_1,\dots,L_{f(n)}\}$, with probability at least $1-2f(n)\cdot \exp(-2\eps^2n)$. Since we assumed the existence of some $L^\star \in \cC$ that attains $\inf_{L \in \cC}\loss_{\cD}(L)$, this $L^\star$ is guaranteed to belong to the interval $\{L_1,L_2,\dots,L_{f(n)}\}$ for sufficiently large $n$. Then, the empirical minimizer returned by the algorithm is guaranteed to have excess error at most $2\eps$. This is because the empirical error of $L^\star$ will be at most $\eps$ off its true error, and the empirical error of any language whose excess error is at least $2\eps$ will also be at most $\eps$ off its true error; so, any such language will not minimize the empirical error.
We thus have that $\iderr(\cA, \cD, \cC, n) \lesssim 2\eps + 2f(n)\cdot \exp(-2\eps^2n)$,
which, upon minimizing $\eps$, yields $\iderr(\cA, \cD, \cC, n) \lesssim \sqrt{(\log(f(n)) + \log(n))/n}$. This is drastically slower than the near-exponential rate that we are targetting.

\begin{algorithm}[t]
    \caption{Weak Identification}\label{alg:weak_identify}
        \begin{algorithmic}[1]
        \Require A sample $\rS = (\rx_1, \dots, \rx_n)$, a function $f:\bN \to \bN$ satisfying $\lim_{n  \to \infty}f(n)=\infty$, a countable language collection $\cC=\{L_1,L_2,\dots\}$
        \State Let $\cA(\rS) $ be the largest index $i\in\left\{  1,\ldots,f{\left(n \right)}\right\} $ such that for all $ j<i $, it holds that
        \[\loss_{\rS}(L_j)-\loss_{\rS}(L_i) > \tfrac{2}{f{\left(n \right)}}\]
        \State \Return $L_{\cA(\rS)}$
    \end{algorithmic}
\end{algorithm}

Towards obtaining this near-exponential rate, we propose \Cref{alg:weak_identify} for weak identification. Instead of returning the minimizer of the empirical error from the window $\{L_1,\dots,L_{f(n)}\}$, the algorithm returns the largest-indexed language that beats every language before it in empirical error by a margin of at least $2/f(n)$. This simple yet crucial modification allows us to restrict the attention solely to the smallest-indexed language in the collection that attains $\inf_{L \in \cC}\loss_{\cD}(L)$; since there are only constantly many languages before this language in the collection, we can guarantee a near-exponential rate for large enough $n$. The formal proof of the following result is deferred to the appendix.

\begin{theorem}[Weak Identification Almost Exponential Rate]
    \label{thm:agnotic-identification-exponential-rate}
    Let $ \cC=\{L_1,L_2,\dots\}$ be a countable collection of languages and $ \cD $ be any distribution over a universe $U$. Suppose there exists $L^{\star}\in \cC$ such that $ \loss_{\cD}(L^\star)=\inf_{L\in \cC} \loss_{\cD}(L)$. Let $f:\bN \to \bN $ be any function satisfying $\lim_{n \to \infty}f(n)=\infty$.
    Then, for large enough $n$ (depending only on $ \cC $, $ \cD $ and $ f $), with probability $ 1-2f(n)\exp{( -n/(2f(n)^2))} $ over $ \rS\sim \cD^{n} $, \Cref{alg:weak_identify} returns $ L_{\cA(\rS)} $ satisfying $\loss_{\cD}(L_{\cA(\rS)})= \inf_{L\in \cC} \loss_{\cD}(L)$. Hence,
    \begin{align*}
        \iderr(\cA, \cD, \cC, n)\le 2f(n)\exp{( -n/(2f(n)^2))}.
    \end{align*}
\end{theorem}

\subsection{Exponential rate optimal when \texorpdfstring{$\inf$}{inf} attained}
\label{sec:exponential-identification-rate-best}

We now show that, under the assumption that $\inf_{L \in \cC}\loss_{\cD}(L)$ is attained within the collection, an exponential rate is the optimal rate that can be achieved by any algorithm.
\begin{restatable}[Weak Identification Exponential Rate Optimal]{theorem}{AgnosticIdentificationExpRateOptimal}
    \label{thm:agnotic-identification-exponential-rate-optimal}
    Let $\cC$ be any collection of languages over a universe $U$ satisfying that there exist $ L,L'\in\cC$ with both $L \setminus (\cup_{\tilde{L}\in \cC,\tilde{L}\neq L} \tilde{L}) \neq \emptyset$ and $L'\setminus (\cup_{\tilde{L}\in \cC,\tilde{L}\neq L'} \tilde{L}) \neq \emptyset$. Then, for any weakidentification algorithm $ \cA $ using randomness $\rr$, there exists a distribution $ \cD $ over $U$ such that there exists $ L^\star\in \cC $ with $ \loss_{\cD}(L^\star)=\inf_{\tilde{L}\in\cC}\loss_{\cD}(\tilde{L}) $, and it further holds for infinitely many $ n $ that
    \begin{align*}
        \iderr(\cA, \cD, \cC, n) \geq \exp(-5n).
    \end{align*}
\end{restatable}
The lower bound instance holds for any collection $\cC$ having two languages $L$ and $L'$ ($ L\not=L' $ ) that contain ``signature'' strings $s_0$ and $s_1$ respectively, such that these strings do not belong to any other language in the collection. We can then construct two distributions $\cD_0$ and $\cD_1$ supported only on these two signature strings, such that $L$ attains $\inf_{\tilde{L} \in \cC}\loss_{\cD_0}(\tilde{L})$, while $L'$ attains $\inf_{\tilde{L} \in \cC}\loss_{\cD_1}(\tilde{L})$. Furthermore, both these distributions produce a common input of $n$ examples with at least an exponential probability. Conditioned on this input, the language output by an algorithm $\cA$ must not equal at least one of $L$ or $L'$: if it doesn't equal $L$, then the algorithm suffers constant excess error under $\cD_0$, and vice versa for $L'$. Randomizing over $\cD_0$ and $\cD_1$ then gives $\iderr(\cA, \cD, \cC, n) \gtrsim e^{-n}$ for infinitely many $n$, either for $\cD=\cD_0$ or $\cD=\cD_1$. The formal proof is deferred to the appendix.

\subsection{Arbitrarily slow rate when \texorpdfstring{$\inf$}{inf} not attained}
\label{sec:lower-bound-inf-not-in-collection}

The next result shows that the assumption about $\inf_{L \in \cC}\loss_{\cD}(L)$ being attained within $\cC$ is necessary not only for obtaining near-exponential rates, but in fact for preventing arbitrarily slow rates a priori.

\begin{restatable}[Arbitrarily Slow Rate When $\inf$ Not Attained]{theorem}{AgnosticIdentificationArbitrarilySlowRate}
    \label{thm:arbitrary-slow-rate-identification}
    There exists a countable collection $ \cC $ over a universe $U$ such that for any weak identification algorithm $ \cA $ using randomness $\rr$, and for any rate function $R:\bN \to (0,1)$ satisfying $\lim_{n \to \infty}R(n)=0$, there exists a distribution $\cD$ for which no $ L^\star\in \cC $ satisfies $\loss_{\cD}(L^\star)= \inf_{L\in\cC}\loss_{\cD}(L) $, and furthermore, it holds
    for infinitely many $ n $ that
    \begin{align*}
        \iderr(\cA, \cD, \cC, n)\geq \frac{R(n)}{8}.
    \end{align*}
\end{restatable}
We provide a proof sketch of the lower bound below; the full proof is deferred to the appendix. Consider a universe of strings consisting of pairs $(w,y)$, where $w \in \bW = \{0,1,2,\dots,\}$, and $y \in \{-1, 0,1\}$. We first specify a common base distribution $\cD_\bW$ over $\bW$ to draw the first component of a pair. Then, for any fixed bit string $z \in \{0\} \times \{0,1\}^\bN$, we consider the distribution $\cD_z$, which first samples $w$ according to $\cD_\bW$, and then outputs the pair $(w,z_w)$. Thus, $z$ can be thought of as specifying the unique $0/1$ labels associated with every $w \in \bW$ (with the label on $w=0$ fixed to be 0); a sample from $\cD_z$ is generated by first drawing $w$ from the marginal distribution $\cD_\bW$, and then ``labeling'' it by the $w^\text{th}$ bit of $z$.

The language collection will have a language $L_I$ for every finite bit string $I \in \{0,1\}^{\ge 1}$. The language $L_I$ will contain $(0,0)$, together with all the pairs $(w, I_w)$ for $1 \le w \le |I|$. It will then be padded with pairs $(w,-1)$ for all $w > |I|$. Namely, each $L_I$ has ``label information'' for only a finite prefix of $\bW$. Then, notice that every $L_I$ necessarily misses out on the label information about an infinite tail of $\bW$; furthermore, as the prefixes $I$ get longer, the mass that any valid distribution $\cD_\bW$ assigns to these tails shrinks to 0. Together, we get that $\loss_{\cD_z}(L_I) > \inf_{L \in \cC}\loss_{\cD_z}(L)=0$ for every $L_I$.

We now consider choosing $\rz$ as a uniformly random bit string in $\{0\} \times \{0,1\}^\bN$. Any finite sample of size $n$ reveals label information about only finitely many $w \in \bW$, and the labels at all other values of $w$ are uniformly random bits. Thus, any algorithm that guesses some language $L_I$ from the collection incurs constant error at all these unseen values of $w$. The main technical argument that remains is to appropriately specify the base distribution $\cD_\bW$: roughly, for the given rate function $R(n)$, we will specify $\cD_\bW$ to be such that it allocates $R(n)$ mass to appropriately-sized blocks of $w$ values. Indeed, the seminal work of \citet{bousquet2021theory} on universal rates in learning gives a construction of such a probability distribution in terms of a pre-specified rate function.
Adopting their construction into our analysis, together with an instantiation of Fatou's lemma, completes the argument.

\subsection{Minimal identification}
Recall our objective for minimal identification (\cref{eqn:minimal-agnostic-identification-benchmark})
\begin{align*}
    &\iderrmin(\cA, \cD, \cC, n) := \p_{\rS \sim \cD^n\hspace{-0.2em},\hspace{0.1em} \rr}\big[ L_{\cA(\rS, \rr)}\not\in \optarg_{\cD,\cC} \text{ or } \exists L\in  \optarg_{\cD,\cC} \text{ s.t. } L\subset L_{\cA(\rS, \rr)}\big].
\end{align*}

We show that a simple modification to \Cref{alg:weak_identify}, which additionally requires the output to have one of its tell-tale sets realized, obtains near-exponential rates for every distribution in the family $\cP(\cC)$ defined in \eqref{eq:minimal-distributions}. Since the algorithm is similar to \Cref{alg:weak_identify}, we defer its presentation (\Cref{alg:minimal_weak_identify} in Appendix~\ref{sec:agnostic-identification-proofs}) and the proof of the following theorem to the appendix.
\begin{theorem}[Agnostic Minimal Identification Near-Exponential Rate]
    \label{thm:agnostic-minimal-identification-exponential-rate}
    Let $ \cC=\{L_1,L_2,\dots\}$ be a countable collection of languages over a countable universe $U$. Let $f:\bN \to \bN $ be any function satisfying $\lim_{n \to \infty}f(n)=\infty$. For any $\cD\in\cP(\cC)$, there exists $ c,C $ such that for $ n $ sufficiently large, \Cref{alg:minimal_weak_identify} --- denoted as $ \cA $ ---  satisfies that
       \begin{align*}
        \iderrmin(\cA, \cD, \cC, n)
        \leq
        cf(n)\exp(-Cn/f(n)^2).
    \end{align*}
\end{theorem}
Recall that the collection $\cP(\cC)$ of distributions only comprises of those distributions in $\cP_{Opt}(\cC)$ for which some optimal language in the collection has a tell-tale contained in the support.
A natural question is whether the existence of a tell-tale set in the support of the distribution is also necessary for minimal identification, or if it suffices that all languages in the collection simply admit a tell-tale set. We show that the former is indeed necessary. Specifically, we establish the following lower bound for deterministic algorithms, showing that indeed their for any algorithm exists a distribution in $\cP_{Opt}(\cC)\setminus \cP(\cC)$, such that the minimal identification error does not vanish. The proof, deferred to the appendix, uses a contradiction similar to the line of arguments in \citet{angluin1988identifying,pitt1989probabilistic,kalavasis2025limits}.

\begin{theorem}\label{thm:necessity-supported-tell-tales}
There exists a countable universe $U$ and a countable language collection
$\cC\subseteq 2^U$ such that every language in $\cC$ admits a finite
tell-tale set and the following holds. Define
\begin{align}
 \bar{\cP}(\cC)=\cP_{Opt}(\cC)\setminus \cP(\cC)
 =
 \{\cD \in \cP_{Opt}(\cC) :
 \forall L \in \optarg_{\cD,\cC},\
 \forall T_L\in \cT_L,\ T_L\not\subseteq \supp(\cD)\}.
\label{eq:necessity-supported-tell-tales-1}
\end{align}
Then, for every deterministic identification algorithm $\bar{\cA}$, there exists
$\cD\in\bar{\cP}(\cC)$ such that
\begin{align*}
\liminf_{n\to \infty}
\p_{\rS \sim \cD^n}\left[
L_{\bar{\cA}(\rS)}\in \optarg_{\cD,\cC},\
\not\exists L\in\optarg_{\cD,\cC}
\text{ such that }L\subsetneq L_{\bar{\cA}(\rS)}
\right]
\leq \frac{1}{2}.
\end{align*}
\end{theorem}

\section{Agnostic language generation}
\label{sec:agnostic-generation}

We will now study agnostic generation; that is, given a sample $ \rS\sim \cD^{n} $ for some distribution $ \cD $ over the universe $ U $, we will aim to construct an algorithm $\cA$, possibly using randomness $\rr$, that generates a string $\cA(\rS, \rr) \in U$ from the unseen support of $ \cD $ with high probability. Namely, the algorithm aims to minimize:
\begin{align*}
    &\generr(\cA, \cD, n) := \e_{\rS \sim \cD^n\hspace{-0.2em},\hspace{0.1em} \rr}\left[\ind\{\cA(\rS,\rr) \notin \supp(\cD) \setminus \rS\}\right] = \p_{\rS\sim \cD^n\hspace{-0.2em},\hspace{0.1em} \rr}[\cA(\rS, \rr) \not\in \supp(\cD)\setminus \rS].
\end{align*}

Note that for the objective above to remain achievable as $n$ gets large, one must assume that $\supp(\cD)=\infty$.

\subsection{General lower bound}
\label{sec:agnostic-generation-lower-bound}

We begin with a result that has a ``no free lunch'' flavor; namely, it shows that the objective in \eqref{eqn:agnostic-generation-benchmark} is intractable in general without further assumptions on the distribution $ \cD $.

\begin{theorem}[Agnostic Generation Lower Bound]
    \label{thm:agnostic-generation-lower-bound}
    Fix any $ 0<\eps <1$. There exists a countable universe $U$, such that for any generation algorithm $ \cA $ using randomness $\rr$, there exists a distribution $ \cD $ over $U$ such that
    \begin{align*}
        \generr(\cA, \cD, n) \geq 1-\eps.
    \end{align*}
\end{theorem}

We note that the lower bound above applies to \textit{every} algorithm, regardless of any collection of languages that the algorithm may have access to. In particular, we may assume that the algorithm always has access to the trivial singleton collection containing the language $L=U$. In this case, the algorithm, in principle, always has access to a language $L$ satisfying $|L \cap \supp(\cD)|=\infty$, but this is not helpful in any meaningful way. This highlights further the issues in introducing a reference language collection $\cC$ into the objective \eqref{eqn:agnostic-generation-benchmark}, as also referenced in the discussion in \Cref{sec:intro-agnostic-generation}. Regardless, as we shall see later, it \textit{is} possible for an algorithm to exploit access to a collection $\cC$ in order to succeed at the objective in \eqref{eqn:agnostic-generation-benchmark} with an \textit{exponential} rate, when the underlying distribution $\cD$ is suitably ``well-behaved'' with respect to the given collection $\cC$.

\subsection{Exponential rate for well-behaved distributions}
\label{sec:agnostic-generation-exponential-rate}

With the general negative result above in mind, we now introduce a condition on the collection $ \cC$ and the distribution $ \cD $ that will allow us to construct a generation algorithm. At a high level, this condition requires every language in the collection not contained within the support of $\cD$ to have a \textit{witness} string among a bounded number of its own first elements.\footnote{We assume that every language in the collection is infinite.} For the purposes of designing an algorithm, this property allows us to eventually rule out every such language, since its witness within this bounded initial segment will never show up in the input.

More formally, recall that we assume the universe of strings $U$ to be a countable set; that is, we can enumerate all the strings in $U$ as $ U=\{u_{1},u_{2},\ldots\} $.
For any language $L\subseteq U$, write
\begin{align*}
    L=\{u_{k_1(L)},u_{k_2(L)},\ldots\},
    \qquad
    k_1(L)<k_2(L)<\cdots,
\end{align*}
for the enumeration of $L$ induced by the enumeration of $U$. Then, for any distribution $ \cD $ over $ U $ and any countable collection of languages $ \cC=\left\{  L_{1}, L_{2},\ldots\right\}  $, where each $L_i \subseteq U $, we define the quantity $i(\cC,\cD)$ as
\begin{align}
    \label{eqn:i-C-D}
        i(\cC,\cD) &:= \min\left\{i \in \bN ~\middle|~ \forall L \in \cC \text{ s.t. } L \nsubseteq \supp(\cD): \exists
        j \le i \text{ s.t. } u_{k_j(L)} \notin \supp(\cD)\right\}.
\end{align}
Here, we define the $\min$ over an empty set to be $\infty$. In words, $i(\cC,\cD)$ is the smallest language-local prefix length such that for any language not fully contained in the support of $ \cD $, one of its first $i(\cC,\cD)$ strings is \textit{not} in the support of $ \cD $.
For instance, notice that such a finite index always exists for any finite collection of languages $ \cC$.

We now show that if $i(\cC, \cD)$ is finite, and there exists at least one language in the collection that is fully contained in the support of $ \cD $, then one can construct a generation algorithm, which with high probability, generates a string from the unseen support of $\cD$.

\begin{theorem}[Agnostic Generation Exponential Rate]
    \label{thm:agnostic-generation-exponential-rate}
    Let $ \cC$ be a countable collection and $\cD$ be a distribution over the universe $U$. If $ i(\cC,\cD) < \infty $, and there exists a language $ L\in \cC$ such that $ L\subseteq \supp(\cD) $,
    then with probability at least $1-c\cdot\exp(-C\cdot n)$ over $ \rS\sim \cD^{n} $ (where $c, C$ are constants depending only on $U, \cD $ and $ \cC$,
    the string $\cA(\rS)$ returned by \Cref{alg:generation} satisfies $\cA(\rS) \in \supp(\cD) \setminus \rS$. Namely,
    $
       \generr(\cA, \cD, n)\leq c\cdot\exp(-C\cdot n).
    $
\end{theorem}

As noted earlier, for a finite collection of languages $ \cC$, it always holds that $ i(\cC,\cD) $ is finite. Thus as a corollary of \Cref{thm:agnostic-generation-exponential-rate}, we have that for any finite collection of languages $ \cC$, if there exists a language $ L\in \cC$ such that $ L\subseteq \supp(\cD) $, then there exists a generation algorithm that generates a string from the unseen support of $ \cD $ with exponentially high probability.
\begin{corollary}
    \label{corollary:agnostic-generation-finite-collections}
    Let $ \cC$ be a finite collection and $\cD$ be a distribution over the universe $U$. If there exists a language $ L\in \cC$ such that $ L\subseteq \supp(\cD) $,
    then with probability at least $1-c\cdot\exp(-C\cdot n)$ over $ \rS\sim \cD^{n} $ (where $c, C$ are constants depending only on $U, \cD $ and $ \cC$), the string $\cA(\rS)$ returned by \Cref{alg:generation} satisfies $\cA(\rS) \in \supp(\cD) \setminus \rS$. Namely,
    $
       \generr(\cA, \cD, n)\leq c\cdot\exp(-C\cdot n).
    $
\end{corollary}
It is natural to ask whether the condition that $i(\cC, \cD)$ is finite is necessary for achieving an exponential rate, also even in the case that there exists a language $L \in \cC$ such that $L \subseteq \supp(\cD)$. Our next result shows that this is indeed the case. Specifically we show that for every deterministic generation algorithm, there exists a instance such that $i(\cC, \cD)$ is infinite, but there exists a language $L \in \cC$ such that $L \subseteq \supp(\cD)$, and the algorithm fails to generate a string from the unseen support of $\cD$ with probability larger than $1/2-\varepsilon$ for any $ \varepsilon>0 $  infinitely often.

\begin{theorem}\label{thm:genlearnabilitycharacterization}
There exists a countable universe $U$, a enummeration of this universe, and a countable language collection $\cC\subseteq 2^U$ such that the following holds. For every deterministic generation algorithm $\bar{\cA}$, there exists a distribution
$\cD\in  \bar{\cP}_{gen}(\cC)
 =
 \{\cD\in \cP(U): \exists L\in \cC,\ L\subseteq \supp(\cD),\ i(\cD,\cC)=\infty \}$ such that
\begin{align*}
    \limsup_{n\to\infty}\generr(\bar{\cA},\cD,n)\geq \frac{1}{2}.
\end{align*}
\end{theorem}

\subsection{Tightness of exponential rate}
\label{sec:agnostic-generation-exponential-rate-lb}

The final result shows that a better-than-exponential rate is impossible for agnostic generation under the assumption that there exists $L \in \cC$ satisfying $L \subseteq \supp(\cD)$.

\begin{restatable}[Agnostic Generation Exponential Rate Lower Bound]{theorem}{AgnosticGenerationExponentialRateLowerBound}
    \label{thm:generation-exponential-rate-lower-bound}
    Let $\cC$ be any collection over a universe $U$ such that there exist languages $ L,L'\in \cC$ with $ |L\cap L'|<\infty $ and $ U \setminus (L\cup L')\not=\emptyset $. For any generation algorithm $ \cA $ using randomness $\rr$, there exists a distribution $ \cD $ over $U$ such that $ \exists L\in \cC$ with $ L\subseteq \supp(\cD) $, and furthermore for infinitely many $ n $
    \begin{align*}
        \generr(\cA, \cD, n) \geq \exp(-2n)/4.
    \end{align*}
\end{restatable}

To get some intuition for the lower bound, consider the special case where the collection has two languages $L, L'$ satisfying $L \cap L'=\emptyset$, and $U \setminus (L \cup L') \neq \emptyset$. In this case, let $s$ be any string that belongs to neither of $L$ or $L'$. Consider two distributions $\cD_0$ and $\cD_1$ that each put mass $1/2$ on $s$, and spread the rest of their mass on $L$ and $L'$ respectively. Under both $\cD_0$ and $\cD_1$, with probability $2^{-n}$, the input comprises solely of $n$ copies of $s$. Conditioned on this input, the output of the algorithm must necessarily not belong to one of $L$ or $L'$, since these are disjoint languages. But if the output does not belong to $L$, then the algorithm errs if the true distribution was $\cD_0$, and vice versa. Randomizing over $\cD_0$ and $\cD_1$ yields the lower bound. The formal proof is given in the appendix.

\section*{Conclusion}
In this paper, we initiated a study of agnostic language identification and generation, where the data distribution may be arbitrary, and not necessarily supported on any language from the reference collection. We proposed natural objectives to study for both the problems. For language identification, we define two objectives weak and minimal identification, and give insights into when and when not these objectives can be achieved and at which rates. For language generation we define a natural objective, and give the same insight as for identification.

More broadly, the setting of agnostic generation is rather nuanced, and, as alluded to earlier, several natural objectives run into pathologies. While the objectives we studied are natural, we believe that there is room for other interesting objectives in the agnostic setting. Thus, future work studying other objectives that meaningfully capture the difficulty of agnostic language generation, we would find interesting.

\section*{Funding}
MMH was supported by an Internationalisation Fellowship from the Carlsberg Foundation. MMH was also
supported by the European Union (ERC, TUCLA, 101125203). Views and opinions expressed are however
those of the author(s) only and do not necessarily reflect those of the European Union or the European
Research Council. Neither the European Union nor the granting authority can be held responsible for them.
Lastly MMH was also supported by Independent Research Fund Denmark (DFF) Sapere Aude Research
Leader grant No. 9064-00068B.
CP was supported by Moses Charikar’s and Gregory Valiant’s Simons Investigator Awards, and a Google PhD Fellowship.
The authors declare no competing interests.
\newpage
\begingroup
\urlstyle{same}
\bibliographystyle{plainnat}
\bibliography{references}
\endgroup

\newpage
\appendix
\section[Proofs for agnostic language identification]{Proofs for statements in \Cref{sec:agnostic-identification}}
\label{sec:agnostic-identification-proofs}

\begin{proof}[Proof of \Cref{thm:agnotic-identification-exponential-rate}]
   Let $ L_{i_{\star}} $ be the language in $ \cC $ with the smallest  index $ i_{\star}$ such that $ \loss_{\cD}(L_{i_{\star}}) = \inf_{L\in \cC} \loss_{\cD}(L)$. Since $\lim_{n \to \infty} f(n)=\infty$, we know that for large enough $n$, $ f(n) \geq i_{\star}$, implying that the algorithm will consider $ L_{i_{\star}} $ as a candidate. Now, since $ i_{\star} $ is the smallest index such that $ \loss_{\cD}(L_{i_{\star}}) = \inf_{L\in \cC} \loss_{\cD}(L)$, we know that for all $ j<i_{\star} $, it holds that  $\loss_{\cD}(L_{j}) > \loss_{\cD}(L_{i_{\star}})$. Since there are only finitely many $j < i$,  there exists some constant $ c_{\gap}>0 $ such that for all $ j<i_{\star} $ it holds that  $\loss_{\cD}(L_{j})-\loss_{\cD}(L_{i_{\star}})   \geq c_{\gap} $.

   Now, by Hoeffding's inequality, for any $ j\in\left\{  1,\ldots,f(n)\right\}  $,
    \begin{align*}
        &\p_{\rS\sim\cD^{n}}\left[\left|\loss_{\rS}(L_{j}) - \loss_{\rS}(L_{i_\star}) -\left(\loss_{\cD}(L_{j}) - \loss_{\cD}(L_{i_\star})\right)\right|\geq \frac{1}{f(n)}\right] \leq  2\exp{\left( -\frac{n}{2f(n)^2}\right)}.
    \end{align*}
    Thus, by the union bound, we have that with probability at least $ 1-2f(n)\exp{\left( -\frac{n}{2f(n)^2}\right)} $, it holds for all $j\in\left\{  1,\ldots,f(n)\right\}  $ that
    \begin{align*}
       &\left|\loss_{\rS}(L_{j}) - \loss_{\rS}(L_{i_\star}) -\left(\loss_{\cD}(L_{j}) - \loss_{\cD}(L_{i_\star})\right)\right|\leq \frac{1}{f(n)}.
    \end{align*}
    Now, on this event, we have that for all $ j<i_{\star} $,
    \begin{align*}
       &\loss_{\rS}(L_{j}) - \loss_{\rS}(L_{i_\star})\geq \left(\loss_{\cD}(L_{j}) - \loss_{\cD}(L_{i_\star})\right)-\frac{1}{f(n)}\geq c_{\gap}-\frac{1}{f(n)}.
    \end{align*}
    But notice also that for $ n $ large enough, we will have that $ c_{\gap} > \frac{3}{f(n)} $, implying that for all $ j<i_{\star} $,
    \begin{align*}
       \loss_{\rS}(L_{j}) - \loss_{\rS}(L_{i_\star}) > \frac{2}{f(n)}.
    \end{align*}
    Thus, \Cref{alg:weak_identify} will consider selecting $i_{\star}$ as $\cA(\rS)$. Now for $ j>i_{\star} $, we have by definition of $i_\star$ that
    \begin{align*}
     \loss_{\cD}(L_j)-\loss_{\cD}(L_{i_\star}) \geq 0.
    \end{align*}
    Thus, for all $ j>i_{\star} $ it holds that
    \begin{align*}
       &\loss_{\rS}(L_{j}) - \loss_{\rS}(L_{i_{\star}})\geq -\frac{1}{f(n)}+\left(\loss_{\cD}(L_{j}) - \loss_{\cD}(L_{i_{\star}})\right)\geq -\frac{1}{f(n)} \\
       \implies\qquad& \loss_{\rS}(L_{i_{\star}}) - \loss_{\rS}(L_{j}) \le \frac{1}{f(n)},
    \end{align*}
    implying that the algorithm will not select $j$ as $\cA(\rS)$ for any $j>i_{\star} $. Thus, we have that the algorithm will return $ L_{i_{\star}} $, with probability at least $ 1-2f(n)\exp{\left( -\frac{n}{2f(n)^2}\right)} $, for large enough $ n $. The law of total expectation then yields the bound on $\iderr(\cA, \cD, \cC, n)$ claimed in the theorem.
\end{proof}

\begin{proof}[Proof of \Cref{thm:agnotic-identification-exponential-rate-optimal}]
    Fix $s_{0}\in L \setminus \left(\bigcup_{\tilde{L}\in \cC,\tilde{L}\neq L} \tilde{L}\right)$ and $s_{1}\in L'\setminus \left(\bigcup_{\tilde{L}\in \cC,\tilde{L}\neq L'} \tilde{L}\right)$. Define distributions $\cD_{0}$ and $\cD_{1}$ as
    \begin{align*}
      &\p_{\rx\sim \cD_{0}}[\rx=s_{0}]=\frac{3}{4},\quad \p_{\rx\sim \cD_{0}}[\rx=s_{1}]=\frac{1}{4};\\
      &\p_{\rx\sim \cD_{1}}[\rx=s_{0}]=\frac{1}{4},\quad \p_{\rx\sim \cD_{1}}[\rx=s_{1}]=\frac{3}{4}.
    \end{align*}
    Then, since $s_{0}$ is not in any language other than $ L $, we have that $\loss_{\cD_0}(\tilde{L}) \ge \frac{3}{4} $ for every $ \tilde{L}\neq L $. Similarly, we also have that $ \loss_{\cD_1}(\tilde{L}) \ge \frac{3}{4} $ for every $ \tilde{L}\neq L' $. Furthermore, we can also observe that
    \begin{align*}
        &\inf_{\tilde{L}\in \cC}\loss_{\cD_{0}}(\tilde{L})=\loss_{\cD_{0}}(L)=\frac{1}{4},
        \\
        &\inf_{\tilde{L}\in \cC}\loss_{\cD_{1}}(\tilde{L})=\loss_{\cD_{1}}(L')=\frac{1}{4}.
    \end{align*}
    Now, consider any identification algorithm $\cA $ using randomness $\rr$, which, upon taking as input a sample $ \rS $ of size $ n $,
    outputs some language $L_{\cA(\rS,\rr)}\in \cC $. For any $ n\in \mathbb{N} $, consider the realization of the input sample $S=(s_{0},\ldots,s_{0},s_{1},\ldots,s_{1}) $, where the first $ \lfloor \frac{n}{2} \rfloor $ examples are equal to $ s_{0} $ and the rest of the $\lceil \frac{n}{2} \rceil$ examples are equal to $ s_{1} $. Now, for any realization of the randomness $\rr$, the language $L_{\cA(s,\rr)}$ is not equal to at least one of $L$ or $L'$. Combined with the observations above, this implies that $\loss_{\cD_{0}}(L_{\cA(S,\rr)})+\loss_{\cD_{1}}(L_{\cA(S,\rr)}) \geq \frac{3}{4} $.
    Using this, we have that for any $n \in \bN$,
    \begin{align*}
        &\max_{b\in\{0,1\}}\e_{\rS\sim \cD_{b}^{n}, \rr}\left[\loss_{\cD_{b}}(L_{\cA(\rS,\rr)})-\inf_{\tilde{L}\in \cC}\loss_{\cD_{b}}(\tilde{L})\right]
        \\
        &\geq
        \frac{1}{2}\left(
        \e_{\rS\sim \cD_{0}^{n}, \rr}\left[\loss_{\cD_{0}}(L_{\cA(\rS,\rr)})-\inf_{\tilde{L}\in \cC}\loss_{\cD_{0}}(\tilde{L})\right]
        +
        \e_{\rS\sim \cD_{1}^{n}, \rr}\left[\loss_{\cD_{1}}(L_{\cA(\rS,\rr)})-\inf_{\tilde{L}\in \cC}\loss_{\cD_{1}}(\tilde{L})\right]
        \right)
        \\
        &\geq
        \frac{1}{2}\left(
        \e_{\rr}\left[\p_{\rS\sim\cD_{0}^{n}}[\rS=S] \cdot \left(\loss_{\cD_{0}}(L_{\cA(S,\rr)})-\inf_{\tilde{L}\in \cC}\loss_{\cD_{0}}(\tilde{L})\right)\right.\right. \\
        &\quad+
        \left.\left.
        \p_{\rS\sim\cD_{1}^{n}}[\rS=S] \cdot \left(\loss_{\cD_{1}}(L_{\cA(S,\rr)})-\inf_{\tilde{L}\in \cC}\loss_{\cD_{1}}(\tilde{L})\right)\right]\right)
        \tag{Law of total expectation and independence of $ \rS,\rr $ }
        \\
        &\geq
        \frac{1}{2}\cdot \left(\frac{1}{4}\right)^{\negmedspace\lceil \frac{n}{2} \rceil}\left(\frac{3}{4}\right)^{\negmedspace\lfloor \frac{n}{2} \rfloor} \negmedspace\negmedspace\cdot
        \e_{\rr}\left[\loss_{\cD_{0}}(L_{\cA(S,\rr)})+\loss_{\cD_{1}}(L_{\cA(S,\rr)})-\inf_{\tilde{L}\in \cC}\loss_{\cD_{0}}(\tilde{L})-\inf_{\tilde{L}\in \cC}\loss_{\cD_{1}}(\tilde{L})\right]
        \tag{since $\p_{\rS\sim\cD_{0}^{n}}[\rS=S], \p_{\rS\sim\cD_{1}^{n}}[\rS=S]\geq (\frac{1}{4})^{\lceil \frac{n}{2} \rceil}(\frac{3}{4})^{\lfloor \frac{n}{2} \rfloor}$ }
        \\
        &\geq
        \frac{1}{2}\cdot\left(\frac{1}{4}\right)^{\lceil \frac{n}{2} \rceil}\left(\frac{3}{4}\right)^{\lfloor \frac{n}{2} \rfloor}\cdot \left(\frac{3}{4}-\frac{1}{2}\right)
        \tag{since $\inf_{\tilde{L}\in \cC}\loss_{\cD_{0}}(\tilde{L})=\inf_{\tilde{L}\in \cC}\loss_{\cD_{1}}(\tilde{L})=\frac{1}{4}$ and $ \loss_{\cD_{0}}(L_{\cA(S,\rr)})+\loss_{\cD_{1}}(L_{\cA(S,\rr)}) \geq \frac{3}{4} $}
        \geq \left(\frac{1}{4}\right)^{n+2} \geq \exp(-5n).
    \end{align*}
    This implies the existence of some $ b\in\{0,1\} $, such that for infinitely many $n$,
    \begin{align*}
        \e_{\rS\sim \cD_{b}^{n}, \rr}\left[\loss_{\cD_{b}}(L_{\cA(\rS,\rr)})-\inf_{\tilde{L}\in \cC}\loss_{\cD_{b}}(\tilde{L})\right] \geq \exp(-5n).
    \end{align*}
\end{proof}

We need the following technical lemma from \cite{bousquet2021theory} to give the proof of \Cref{thm:arbitrary-slow-rate-identification}.
\begin{lemma}[Lemma 5.12 in \cite{bousquet2021theory}]
    \label{lemma:sequence}
    Let $R:\bN \to (0,1)$ be any rate function satisfying $\lim_{n \to \infty}R(n)=0$. There exists a probability distribution $p$ over $\bN$, two increasing sequences of natural numbers $(n_i)_{i \in \bN}$ and $(k_i)_{i \in \bN}$, and a constant $\frac12 \le C \le 1$ such that the following hold for all $i \ge 1$:
    \begin{enumerate}
        \item $\sum_{k > k_i}p_k \le \frac{1}{n_i}$.
        \item $n_ip_{k_i} \le k_i$.
        \item $p_{k_i}=C\cdot R(n_i) > 0$.
    \end{enumerate}
\end{lemma}
We are now ready to prove \Cref{thm:arbitrary-slow-rate-identification}.

\begin{proof}[Proof of \Cref{thm:arbitrary-slow-rate-identification}]
    The universe $U$ for the collection $\cC$ will be $U= \bW \times \{-1,0,1\}$. For any $i \in \bN$ and any $I \in \{0,1\}^i$, consider the language
    \begin{align*}
         L_{I}=\{(0,0)\}\bigcup\left(\bigcup_{j=1}^{i} \left\{  (j,I_{j}) \right\}\right)\bigcup\left(\bigcup_{j>i}\left\{  (j,-1)\right\} \right).
    \end{align*}
    Now define the collection $\cC$ as
    \begin{align*}
         \cC=\{L_I : i\in\bN,\ I\in\{0,1\}^{i}\}.
    \end{align*}
    We first observe that the collection $ \cC $ is countable, as it is a countable union of finitely many languages.

    Now, given $ R $, consider the probability distribution $p$ and sequences $(n_i)_{i \in \bN}$ and $(k_i)_{i \in \bN}$ guaranteed by \Cref{lemma:sequence}. We will now define a base distribution $\cD_\bW$ over $\bW$. For this, define $\sigma_0=0$, and for $i \ge 1$, define $\sigma_i=\sum_{j=1}^i k_j$. Now define $\cD_\bW$ as follows, where for every $w \in \bN$,
    \begin{align*}
        \p_{\rw\sim \cD_{\mathbb{W}}}[\rw=w] = \frac{p_{k_i}}{2k_i}, \quad \text{ where $i$ is such that } 2\sigma_{i-1}+1 \le w \le 2\sigma_i,
    \end{align*}
    and
    \begin{align*}
        \p_{\rw\sim \cD_{\mathbb{W}}}[\rw=0] = 1-\sum_{i=1}^{\infty}p_{k_i}.
    \end{align*}
    In words, $\cD_\bW$ assigns mass $p_{k_1}/2k_1$ each to the first $2k_1$ natural numbers, $p_{k_2}/2k_2$ mass each to the next $2k_2$ natural numbers, and so on, and the remaining mass on $0$.
    Now, for any sequence $z\in\{0\} \times \{0,1\}^\bN$, define the distribution $ \cD_{z} $ over the universe $U$ as
    \begin{align*}
        \p_{\rx\sim \cD_{z}}[\rx=(w,y)]=\begin{cases}
        \p_{\rw\sim \cD_{\bW}}[\rw=w] & \text{ if } y=z_{w}, \\
        0 & \text{ otherwise.}
        \end{cases}
    \end{align*}
    In other words, to generate a sample $\rx$ from $ \cD_{z} $, we first draw $ \rw \sim  \cD_\bW$, and then output the string $\rx=(\rw,z_{\rw})$. This alternative view of sampling from $ \cD_{z} $ will be useful going forward. In particular, a sample $\rS \sim \cD_z^n$ can instead be viewed as the draw $\rS=(\rW, z_\rW)$, where $\rW \sim \cD_\bW^n$ and $z_\rW$ corresponds to the bits in $z$ at indices in $\rW$.

    Now, we observe that for any $z$, $\inf_{L \in \cC} \p_{\rx \sim \cD_z}[\rx \notin L]=0$, but no $L^\star \in \cC$ satisfies $\p_{\rx \sim \cD_z}[\rx \notin L^\star]=0$. To see this, fix any $z$. Since every $L_I \in \cC$ contains the element $(0,0)$ and $z_0=0$ by construction,
    the error is only determined by the bits $z_w$ where $w > 0$. Then, for any $L_I \in \cC$, note that $L_I$ only contains strings $(j,-1)$ for $j > |I|$, and so, $\p_{\rx \sim \cD_z}[\rx \notin L_I] \ge \sum_{j > |I|}\p_{\rw\sim \cD_{\mathbb{W}}}[\rw=j] > 0$. But now, for any $i$, consider $L_I$, where $I=z_{1:2\sigma_i}$. Then, $\p_{\rx \sim \cD_z}[\rx \notin L_I] = \sum_{j > i} p_{k_j} \le \sum_{k > k_i} p_{k} \le 1/n_i$. Thus, $\inf_{L \in \cC} \p_{\rx \sim \cD_z}[\rx \notin L]=0$, but the infinimum is not achieved by any language in the collection.

    Now, let $\cA$ be any (possibly randomized) identification algorithm using randomness $\rr$, which, upon taking as input a sample $ \rS $ of size $ n $ 
    outputs some language $L_{\cA(\rS, \rr)}\in \cC $. Consider drawing a uniformly random sequence $ \rz\in\{0\}\times\{0,1\}^{\mathbb{N}} $; it then holds that

    \begingroup
    \allowdisplaybreaks
    \begin{align*}
        &\e_{\rz}\left[\limsup_{n\rightarrow\infty} \frac{ 1}{C\cdot R(n)}\e_{\rS\sim\cD_{\rz}^{n}, \rr}\left[\loss_{\cD_{\rz}}(L_{\cA(\rS, \rr)})\right]\right]
        \\
        &\ge \e_{\rz}\left[\limsup_{i\rightarrow\infty} \frac{ 1}{C\cdot R(n_i)}\e_{\rS\sim\cD_{\rz}^{n_i}, \rr}\left[\loss_{\cD_{\rz}}(L_{\cA(\rS, \rr)})\right]\right] \\
        =& \e_{\rz}\left[\limsup_{i\rightarrow\infty}\frac{ 1}{C\cdot R(n_i)} \e_{\rW\sim\cD_{\bW}^{n_i}, \rr}\left[\p_{\rw\sim\cD_{\bW}}[(\rw,\rz_{\rw})\notin L_{\cA(\rW,\rz_\rW, \rr)}]\right]\right] \tag{alternative view of sampling}\\
        =& \e_{\rz}\left[\limsup_{i\rightarrow\infty}\frac{ 1}{C\cdot R(n_i)} \e_{\rw\sim\cD_{\bW}}\left[\p_{\rW\sim\cD_{\bW}^{n_i}, \rr}[(\rw,\rz_{\rw})\notin L_{\cA(\rW,\rz_\rW, \rr)}]\right]\right] \tag{$\rw$, $\rr$, $\rW$ all independent}\\
        \ge& \e_{\rz}\left[\limsup_{i\rightarrow\infty}\frac{ 1}{C\cdot R(n_i)} \cdot \frac{p_{k_i}}{2k_i}\sum_{w=2\sigma_{i-1}+1}^{2\sigma_i}\p_{\rW\sim\cD_{\bW}^{n_i}, \rr}[(w,\rz_{w})\notin L_{\cA(\rW,\rz_\rW, \rr)}]\right] \\
        \ge& \limsup_{i\rightarrow\infty}\e_{\rz}\left[\frac{ 1}{C\cdot R(n_i)} \cdot \frac{p_{k_i}}{2k_i}\sum_{w=2\sigma_{i-1}+1}^{2\sigma_i}\p_{\rW\sim\cD_{\bW}^{n_i}, \rr}[(w,\rz_{w})\notin L_{\cA(\rW,\rz_\rW, \rr)}]\right] \tag{Reverse Fatou's Lemma}\\
        =& \limsup_{i\rightarrow\infty}\frac{ 1}{2k_i}\sum_{w=2\sigma_{i-1}+1}^{2\sigma_i} \e_{\rW\sim\cD_{\bW}^{n_i}, \rr}\left[\p_{\rz}[(w,\rz_{w})\notin L_{\cA(\rW, \rz_\rW, \rr)}]\right] \tag{since $p_{k_i}=C\cdot R(n_i)$, and $\rW, \rz, \rr$ independent}\\
        \ge&  \limsup_{i\rightarrow\infty}\frac{ 1}{2k_i}\sum_{w=2\sigma_{i-1}+1}^{2\sigma_i}\e_{\rr}\left[ \p_{\rW \sim \cD_\bW^{n_i}}[w \notin \rW] \cdot \e_{\rW\sim\cD_{\bW}^{n_i}}\left[\p_{\rz}[(w,\rz_{w})\notin L_{\cA(\rW, \rz_\rW, \rr)}] ~\middle|~ w \notin \rW\right]\right] \\
        \ge& \limsup_{i\rightarrow\infty}\frac{ 1}{2k_i}\sum_{w=2\sigma_{i-1}+1}^{2\sigma_i}\e_{\rr}\left[ \p_{\rW \sim \cD_\bW^{n_i}}[w \notin \rW] \cdot \frac12\right] \tag{$\star$}\\
        =& \limsup_{i\rightarrow\infty}\frac{ 1}{2k_i}\sum_{w=2\sigma_{i-1}+1}^{2\sigma_i}\left[ \left(1-\frac{p_{k_i}}{2k_i}\right)^{n_i} \cdot \frac12\right]
        = \limsup_{i\rightarrow\infty}\frac{ 1}{4k_i}\cdot 2k_i \cdot  \left(1-\frac{p_{k_i}}{2k_i}\right)^{n_i}  \\
        \ge& \frac12\lim_{i\rightarrow\infty}
\left(1-\frac{1}{2n_i}\right)^{n_i}
\tag{since $n_ip_{k_i}\le k_i$}\\
=& \frac{e^{-1/2}}{2}>\frac14.
    \end{align*}
    \endgroup

    In the above, we could apply Reverse Fatou's Lemma since
    \begin{align*}
        \frac{ 1}{C\cdot R(n_i)} \cdot \frac{p_{k_i}}{2k_i}\sum_{w=2\sigma_{i-1}+1}^{2\sigma_i}\p_{\rW\sim\cD_{\bW}^{n_i}, \rr}[(w,\rz_{w})\notin L_{\cA(\rW,\rz_\rW, \rr)}]
        \le \frac{1}{2k_i} \cdot 2k_i \le 1.  \tag{using $p_{k_i}=C\cdot R(n_i)$}
    \end{align*}
    Furthermore, at step $(\star)$, we used that, once $\rr$ and $\rW$ have been realized, $\rz_\rW$ is independent of $\rz_w$, conditioned on $w \notin \rW$, and hence $(w,\rz_w)$ for $ w\not=0 $  is not in $L_{\cA(\rW,\rz_\rW, \rr)}$ with probability at least $1/2$.

    The above analysis implies that there exists $z \in \{0\}\times\{0,1\}^\bN$ such that for infinitely many $n$,
    \begin{align*}
        \e_{\rS\sim\cD_{z}^{n},\rr}\left[\loss_{\cD_{z}}(L_{\cA(\rS, \rr)})\right] = \e_{\rS\sim\cD_{z}^{n},\rr}\left[\loss_{\cD_{z}}(L_{\cA(\rS, \rr)})-\inf_{L\in\cC}\loss_{\cD_z}(L)\right] \ge \frac{C \cdot R(n)}{4} \ge \frac{R(n)}{8}. \tag{since $C \ge 1/2$}
    \end{align*}
\end{proof}

\begin{algorithm}[H]
    \caption{Agnostic Minimal Identification}\label{alg:minimal_weak_identify}
    \begin{algorithmic}[1]
        \Require A sample $\rS = (\rx_1, \dots, \rx_n)$, a function $f:\bN \to \bN$ satisfying $\lim_{n \to \infty}f(n)=\infty$, a countable language collection $\cC=\{L_1,L_2,\dots\}$, and a computable mapping generating tell-tale sets for each $L\in \cC$.
        \State Let $\cA(\rS)$ be the \textit{smallest} index $i\in\left\{1,\ldots,f(n)\right\}$ such that:
        \Statex \quad 1) For all $j \in \left\{1,\ldots,f(n)\right\}$, it holds that $\loss_{\rS}(L_j)-\loss_{\rS}(L_i) > -\frac{2}{f(n)}$.
        \Statex \quad 2) There exists a finite tell-tale set $T_{L_i} \in \cT_{L_i}$ for $L_i$ such that $T_{L_i} \subseteq \rS$.
        \State \Return $L_{\cA(\rS)}$ (if no such $i$ exists, return the language with the smallest index only satisfying the first condition).
    \end{algorithmic}
\end{algorithm}
The empirical-error comparison involves only the first $ f(n) $ languages in the collection, and each only takes n evaluations, given that we have query access to if an element is in a language. Furthermore, given a query access to the tell-tale sets, the algorithm can check whether a tell-tale set is contained in the sample, in finitely many checks. Thus, given these assumption the algorithm is computable in finitely many operations.
\begin{proof}[Proof of \Cref{thm:agnostic-minimal-identification-exponential-rate}]
    Let $ i_{\star} $ be the smallest index such that $L_{i_\star}$ is optimal and admits a tell-tale set $T_{i_\star}\in\cT_{L_{i_\star}}$ with $T_{i_\star}\subseteq \operatorname{supp}(\cD)$. Such an index exists by the definition of $\cP(\cC)$, and is well-defined since $\cC$ is countable. If $T_{i_\star}\neq\emptyset$, let $p_{\min}=\min_{x\in T_{i_\star}}\p_{\rx\sim\cD}[\rx=x]>0$. We will show that $L_{i_\star}$ is selected by \Cref{alg:minimal_weak_identify} with high probability, and that it is minimal.
    First, we verify the minimality of $L_{i_\star}$. Fix any $L'\in\cC$ satisfying $L'\subsetneq L_{i_\star}$. If $T_{i_\star}=\emptyset$, then the tell-tale condition $T_{i_\star}\nsubseteq L'$ rules out the existence of such an $L'$. Otherwise, since $T_{i_\star}$ is a tell-tale set for $L_{i_\star}$, there exists some $x\in T_{i_\star}$ with $x\notin L'$, i.e. $ x\in L_{i_{\star}}\setminus L' $ . Since $T_{i_\star}\subseteq\operatorname{supp}(\cD)$ and $U$ is countable, $\p_{\rx\sim\cD}[\rx=x]>0$. Hence, since $ L'\subset L_{i_{\star}} $ implying $U\setminus L'=(U\setminus L_{i_\star}) \cup (L_{i_{\star}}\setminus L')$ we get,
    \begin{align*}
        \loss_{\cD}(L')=\loss_{\cD}(L_{i_\star})+\p_{\rx\sim\cD}[\rx\in L_{i_\star}\setminus L']>\loss_{\cD}(L_{i_\star}).
    \end{align*}
    Thus no proper sublanguage of $L_{i_\star}$ is also optimal.

    Since $\lim_{n\to\infty}f(n)=\infty$, for all sufficiently large $n$ the window $\{1,\ldots,f(n)\}$ contains $i_\star$. Also define
    \begin{align*}
        \gamma=\min\left\{\loss_{\cD}(L_i)-\loss_{\cD}(L_{i_\star}): i<i_\star,\ \exists T_i\in\cT_{L_i}\text{ with }T_i\subseteq\operatorname{supp}(\cD)\right\},
    \end{align*}
    taking $\gamma=1$ if the set above is empty. To see that $\gamma>0$, recall that $i_\star$ was chosen as the smallest index whose language is optimal and has a tell-tale set contained in $\operatorname{supp}(\cD)$. Therefore, any index $i<i_\star$ appearing in the minimum above already has a tell-tale set contained in $\operatorname{supp}(\cD)$, so it cannot also be optimal; otherwise it would contradict the minimal choice of $i_\star$. Thus each such index is suboptimal, and since there are only finitely many indices $i<i_\star$, the minimum gap is strictly positive. We now further take $n$ large enough that $1/f(n)<\gamma/3$.

   By Hoeffding's inequality, for any $j\in\{1,\ldots,f(n)\}$,
    \begin{align*}
        \p_{\rS\sim\cD^n}\left[\left|\loss_{\rS}(L_j)-\loss_{\rS}(L_{i_\star})-\left(\loss_{\cD}(L_j)-\loss_{\cD}(L_{i_\star})\right)\right|\geq \frac{1}{f(n)}\right]
        \leq 2\exp{\left(-\frac{n}{2f(n)^2}\right)}.
    \end{align*}
    Therefore, with probability at least $1-2f(n)\exp{\left(-n/(2f(n)^2)\right)}$, the event
    \begin{align}\label{eq:agnostic-minimal-identification-exponential-rate}
        \left|\loss_{\rS}(L_j)-\loss_{\rS}(L_{i_\star})-\left(\loss_{\cD}(L_j)-\loss_{\cD}(L_{i_\star})\right)\right|\leq \frac{1}{f(n)}
    \end{align}
    holds simultaneously for every $j\in\{1,\ldots,f(n)\}$.

    On this event, for every $j\in\{1,\ldots,f(n)\}$, since $\loss_{\cD}(L_{i_\star})= \inf_{L\in \cC}\loss_{\cD}(L) $, we have
    \begin{align*}
        \loss_{\rS}(L_j)-\loss_{\rS}(L_{i_\star})
        \geq \loss_{\cD}(L_j)-\loss_{\cD}(L_{i_\star})-\frac{1}{f(n)}
        \geq -\frac{1}{f(n)}
        >-\frac{2}{f(n)}.
    \end{align*}
    Thus $L_{i_\star}$ satisfies the first condition of \Cref{alg:minimal_weak_identify}. If $T_{i_\star}=\emptyset$, then the tell-tale condition holds deterministically. Otherwise, since every element of $T_{i_\star}$ has probability at least $p_{\min}$, a union bound gives
    \begin{align*}
        \p_{\rS\sim\cD^n}\left[T_{i_\star}\nsubseteq \rS\right]
        \leq \sum_{x\in T_{i_\star}}(1-\p_{\rx\sim\cD}[\rx=x])^n
        \leq |T_{i_\star}|\exp(-p_{\min}n).
    \end{align*}
    Hence, outside this second failure event, $L_{i_\star}$ also satisfies the tell-tale condition.

    It remains to show that no smaller index can be selected. Fix $k<i_\star$. If condition 2 holds for $L_k$, then the sampled tell-tale set for $L_k$ is contained in $\operatorname{supp}(\cD)$, almost surely since $U$ is countable, and so by the definition of $\gamma$,
    \begin{align*}
        \loss_{\cD}(L_k)-\loss_{\cD}(L_{i_\star})\geq\gamma.
    \end{align*}
    On the concentration event \cref{eq:agnostic-minimal-identification-exponential-rate}, this implies
    \begin{align*}
        \loss_{\rS}(L_{i_\star})-\loss_{\rS}(L_k)
        \leq -\gamma+\frac{1}{f(n)}
        <-\frac{2}{f(n)}.
    \end{align*}
    Since $i_\star\leq f(n)$, condition 1 fails for the candidate $k$ when compared against $j=i_\star$. Thus no $k<i_\star$ can satisfy both conditions. Since $i_\star$ does satisfy both conditions, and the algorithm chooses the smallest such index, \Cref{alg:minimal_weak_identify} returns $L_{i_\star}$ on the intersection of the two high-probability events.

    The displayed success-probability bound follows by a union bound over the concentration failure and tell-tale sampling failure events.
\end{proof}

\begin{proof}[Proof of \Cref{thm:necessity-supported-tell-tales}]
Let
$
    U=\{1,2\}\times\mathbb{N},
$
and let $\{p_1,p_2,\ldots\}$ be the increasing enumeration of the primes. For
$i,j\geq 1$, write
$
    a_{i,j}=(1,p_i^j),
    q_i=(2,p_i),$ and $
    q_\infty=(2,1).
$
Define
$
    B_i=\{a_{r,j}:1\leq r\leq i,\ j\geq 1\}$ and $
    Q_i=\{q_1,\ldots,q_i\},
$
with $B_0=Q_0=\emptyset$, and let
$
    B_\infty=\bigcup_{i=1}^{\infty}B_i$ and $
    Q_\infty=\{q_i:i\geq 1\}.
$
For $i,m\geq 1$, set
$
    K_{i,m}
    =
    B_{i-1}\cup\{a_{i,1},\ldots,a_{i,m}\}\cup Q_{i-1}.
$
For $i\geq 1$, set
$
    K_i=B_i\cup Q_i,
$
and finally set
$
    K_\infty=B_\infty\cup Q_\infty\cup\{q_\infty\}.
$
Let
$
    \cC
    =
    \{K_{i,m}:i,m\geq 1\}
    \cup
    \{K_i:i\geq 1\}
    \cup
    \{K_\infty\}.
$
This class is countable.

Every language in $\cC$ has a finite tell-tale. For $K_{1,1}$ we may take
$\{a_{1,1}\}$. For $K_{i,m}\neq K_{1,1}$ we may take $\{a_{i,m}\}$, for $K_i$ we
may take $\{q_i\}$, and for $K_\infty$ we may take $\{q_\infty\}$.

We will use the following stronger fact only for the limiting languages:
every tell-tale for $K_i$ must contain $q_i$, because any finite
$T\subseteq K_i$ omitting $q_i$ is contained in $K_{i,m}\subsetneq K_i$ for
some $m$. Likewise, every tell-tale for $K_\infty$ must contain $q_\infty$,
because any finite $T\subseteq K_\infty$ omitting $q_\infty$ is contained in
$K_i\subsetneq K_\infty$ for some $i$.
For $j\in\mathbb{N}\cup\{\infty\}$, let
\begin{align*}
    K_j^\star=
    \begin{cases}
        K_j, & j<\infty,\\
        K_\infty, & j=\infty.
    \end{cases}
\end{align*}
If a distribution $\cD$ has support $\supp(\cD)=B_j$, then the optimal languages are
precisely the members of $\cC$ containing $B_j$. Among these, the unique
inclusion-minimal optimal language is $K_j^\star$. Thus
\begin{align}
&\left\{
L_{\bar{\cA}(\rS)}\in \optarg_{\cD,\cC},\
\not\exists L\in\optarg_{\cD,\cC}\text{ such that }L\subsetneq L_{\bar{\cA}(\rS)}
\right\}=
\{L_{\bar{\cA}(\rS)}=K_j^\star\}.
\label{eq:necessity-supported-tell-tales-5}
\end{align}
Moreover, every distribution with support $B_j$ belongs to
$\bar{\cP}(\cC)$ as defined in \cref{eq:necessity-supported-tell-tales-1}.
If $j<\infty$, then every tell-tale for $K_{j}^\star=K_j$ contains $q_j$, while
$q_j\notin B_j$. If $L\in\optarg_{\cD,\cC}$ properly contains $K_j$, then any
finite $T\subseteq B_j$ is contained in $K_j\subsetneq L$, so $T$ is not a
tell-tale for $L$. If $j=\infty$, then $\optarg_{\cD,\cC}=\{K_\infty\}$ and
every tell-tale for $K_\infty$ contains $q_\infty\notin B_\infty$.

We next record the deterministic adversarial fact. There is no deterministic
algorithm $\cA$ such that, for every sequence $x=(x_1,x_2,\ldots)$ whose range
is $B_j$ for some $j\in\mathbb{N}\cup\{\infty\}$, there exists
$M_x\in\mathbb{N}$ such that
\begin{align}
    \forall n>M_x,\qquad
    L_{\cA(x_1,\ldots,x_n)}=K_j^\star.
\label{eq:necessity-supported-tell-tales-6}
\end{align}
To see this, assume that such $\cA$ exists. Construct a sequence in phases.
Within $B_i$, use the ordering
\begin{align*}
    a_{1,1},a_{2,1},\ldots,a_{i,1},
    a_{1,2},a_{2,2},\ldots,a_{i,2},\ldots .
\end{align*}
In phase $i$, present the least not-yet-seen element of $B_i$ until both the
finite block
$
    F_i=\{a_{r,s}:1\leq r\leq i,\ 1\leq s\leq i\}
$
has appeared and $\cA$ outputs $K_i$. If this never happens, continue forever
by enumerating the remaining elements of $B_i$ in the same ordering. If the
construction stalls in phase $i$, the final range is $B_i$, but after $F_i$ has
appeared the algorithm never outputs $K_i$, a contradiction. If every phase
ends, then the final range is $B_\infty$: indeed, every $a_{r,s}$ belongs to
$F_i$ for all $i\geq \max\{r,s\}$. At the end of each phase $i$, the learner
outputs $K_i\neq K_\infty$, so it is not eventually equal to $K_\infty$, again
a contradiction.

Now assume for contradiction that there is a deterministic statistical learner
$\bar{\cA}$ such that, for every $\cD\in\bar{\cP}(\cC)$,
\begin{align}
\liminf_{n\to \infty}
\p_{\rS \sim \cD^n}\left[
L_{\bar{\cA}(\rS)}\in \optarg_{\cD,\cC},\
\not\exists L\in\optarg_{\cD,\cC}
\text{ such that }L\subsetneq L_{\bar{\cA}(\rS)}
\right]
> \frac{1}{2}.
\label{eq:necessity-supported-tell-tales-7}
\end{align}
We use $\bar{\cA}$ to build a deterministic sequence learner $\cA$, contradicting
\cref{eq:necessity-supported-tell-tales-6}.

For $n=1$, define $\cA$ arbitrarily. For $n\geq 2$, given a prefix
$(x_1,\ldots,x_n)$ and a string
$\sigma\in\{-1,1\}^{n(n-1)}$, split $\sigma$ into $n$ consecutive blocks of
length $n-1$. For each $r\in[n]$, define
$
    C_r^{(n)}(\sigma)
    =
    \min\{k\in[n-1]:\sigma_{(r-1)(n-1)+k}=-1\},
$
with the convention that $C_r^{(n)}(\sigma)=n$ if the set is empty. Define
$\cA$ by majority over all such resampling strings:
\begin{align}
    L_{\cA(x_1,\ldots,x_n)}
    =
    \argmax_{L\in\cC}
    \sum_{\sigma\in\{-1,1\}^{n(n-1)}}
    \ind\left\{
        L_{\bar{\cA}(x_{C_1^{(n)}(\sigma)},\ldots,x_{C_n^{(n)}(\sigma)})}=L
    \right\},
\label{eq:necessity-supported-tell-tales-8}
\end{align}
breaking ties according to a fixed enumeration of $\cC$.

Fix any sequence $x=(x_1,x_2,\ldots)$ with range $B_j$. Define $\cD_x$ by first
drawing $\tilde C$ with $\p[\tilde C=t]=2^{-t}$ and then outputting
$x_{\tilde C}$. Then $\supp(\cD_x)=B_j$, hence $\cD_x\in\bar{\cP}(\cC)$, and
\cref{eq:necessity-supported-tell-tales-5,eq:necessity-supported-tell-tales-7}
give
\begin{align}
\liminf_{n\to\infty}
\p\left[
    L_{\bar{\cA}(x_{\tilde C_1},\ldots,x_{\tilde C_n})}=K_j^\star
\right]
>\frac{1}{2},
\label{eq:necessity-supported-tell-tales-9}
\end{align}
where $\tilde C_1,\ldots,\tilde C_n$ are i.i.d. copies of $\tilde C$.

For uniform $\sigma$, the variables
$C_1^{(n)}(\sigma),\ldots,C_n^{(n)}(\sigma)$ have the same distribution as
$\min\{\tilde C_1,n\},\ldots,\min\{\tilde C_n,n\}$. Therefore,
\begin{align}
&\p_{\sigma}\left[
    L_{\bar{\cA}(x_{C_1^{(n)}(\sigma)},\ldots,x_{C_n^{(n)}(\sigma)})}=K_j^\star
\right]
\notag\\
&\qquad=
\p\left[
    L_{\bar{\cA}(x_{\min\{\tilde C_1,n\}},\ldots,
    x_{\min\{\tilde C_n,n\}})}=K_j^\star
\right]
\tag{by the distribution of $C_1^{(n)},\ldots,C_n^{(n)}$}\\
&\qquad\geq
\p\left[
    L_{\bar{\cA}(x_{\tilde C_1},\ldots,x_{\tilde C_n})}=K_j^\star,\
    \forall r\in[n]:\tilde C_r\leq n-1
\right]
\tag{on this event, truncation changes nothing}\\
&\qquad\geq
\p\left[
    L_{\bar{\cA}(x_{\tilde C_1},\ldots,x_{\tilde C_n})}=K_j^\star
\right]
-
\p\left[\exists r\in[n]:\tilde C_r>n-1\right]
\tag{by $P[A\cap B]\geq P[A]-P[B^c]$}\\
&\qquad\geq
\p\left[
    L_{\bar{\cA}(x_{\tilde C_1},\ldots,x_{\tilde C_n})}=K_j^\star
\right]
-n2^{-n+1}
\tag{by the union bound}
\\
\Rightarrow &\p_{\sigma}\left[
    L_{\bar{\cA}(x_{C_1^{(n)}(\sigma)},\ldots,x_{C_n^{(n)}(\sigma)})}=K_j^\star
\right]\geq\p\left[
    L_{\bar{\cA}(x_{\tilde C_1},\ldots,x_{\tilde C_n})}=K_j^\star
\right]
-n2^{-n+1}\label{eq:necessity-supported-tell-tales-10}
\end{align}
By \cref{eq:necessity-supported-tell-tales-9,eq:necessity-supported-tell-tales-10},
the right-hand side is eventually strictly larger than $1/2$. Hence, for all
sufficiently large $n$, a strict majority of resampling strings make $\bar{\cA}$
output $K_j^\star$, and by the definition of $\cA$ in
\cref{eq:necessity-supported-tell-tales-8},
$
    L_{\cA(x_1,\ldots,x_n)}=K_j^\star.
$
This holds for every sequence with range $B_j$, contradicting the deterministic
adversarial fact \cref{eq:necessity-supported-tell-tales-6}. Therefore, for every deterministic learner $\bar{\cA}$ there exists
$\cD\in\bar{\cP}(\cC)$ such that
\begin{align*}
\liminf_{n\to \infty}
\p_{\rS \sim \cD^n}\left[
L_{\bar{\cA}(\rS)}\in \optarg_{\cD,\cC},\
\not\exists L\in\optarg_{\cD,\cC}
\text{ such that }L\subsetneq L_{\bar{\cA}(\rS)}
\right]
\leq \frac{1}{2}.
\end{align*}
This concludes the proof.
\end{proof}

\section[Proofs for agnostic language generation]{Proofs for statements in \Cref{sec:agnostic-generation}}

\begin{proof}[Proof of \Cref{thm:agnostic-generation-lower-bound}]
    Define the universe $U$ to be the set $\bN \times \{1,2,\dots,\lceil 1/\eps\rceil\}$. For any sequence $ z $  in $ \{  1,\ldots, \lceil 1/\eps \rceil\}^{\mathbb{N}} $, define a distribution $ \cD_{z} $ over the universe as
    \begin{align*}
     \p_{\rx\sim \cD_{z}}[\rx=(w,y)]=\begin{cases}
        2^{-w} & \text{if } y=z_{w}, \\
        0 & \text{otherwise.}
     \end{cases}
    \end{align*}
    We may alternatively view the process of sampling from $\cD_z$ as first drawing $ \rw $ from a base distribution $\cD_\bN$ over $ \mathbb{N} $ defined as $ \p_{\rw\sim \cD_{\mathbb{N}}}[\rw=w]=2^{-w} $, and then outputting the string $ (\rw,z_{\rw}) $.
    Viewed this way, a sample $\rS \sim \cD_z^n$ can instead be written as $\rS=(\rW, z_\rW)$, where $\rW \sim \cD_\bN^n$ and $z_\rW$ corresponds to the bits in $z$ at indices in $\rW$.

    Now, let $\cA$ be any generation algorithm using randomness $\rr$, which, upon taking as input a sample $ \rS $ of size $ n $ outputs a string $\cA(\rS, \rr) \in U$. For a uniformly random $ \rz\in\{1,2,\dots,\lceil 1/\eps\rceil\}^{\mathbb{N}} $, it holds that
    \begingroup
    \allowdisplaybreaks
    \begin{align*}
     &\e_{\rz}\left[\limsup_{n} \p_{\rS\sim \cD_{\rz}^n\hspace{-0.2em},\hspace{0.1em} \rr}[\cA(\rS, \rr) \not\in \supp(\cD_{\rz})\setminus \rS]\right]
     \geq \limsup_{n} \e_{\rz}\left[\p_{\rS\sim \cD_{\rz}^n\hspace{-0.2em},\hspace{0.1em} \rr}[\cA(\rS, \rr) \not\in \supp(\cD_{\rz})\setminus\rS]\right] \tag{Reverse Fatou's Lemma}\\
     =&\limsup_{n} \e_{\rz}\left[\p_{\rW\sim \cD_{\mathbb{N}}^n\hspace{-0.2em},\hspace{0.1em} \rr}[\cA(\rW, z_{\rW}, \rr) \not\in \supp(\cD_{\rz})\setminus(\rW, z_{\rW}))\right] \tag{alternative sampling view}\\
     =&\limsup_{n} \e_{\rW\sim \cD_{\mathbb{N}}^n\hspace{-0.2em},\hspace{0.1em} \rr}\left[\p_{\rz}\left[\cA(\rW, z_{\rW}, \rr) \not\in \supp(\cD_{\rz})\setminus (\rW, z_{\rW})\right]\right] \tag{$\rr$ independent of input and $\rz$}\\
     =&\limsup_{n}\e_{\rW\sim \cD_{\mathbb{N}}^n\hspace{-0.2em},\hspace{0.1em} \rr}\left[\e_{\rz_{\rW}}\left[\p_{\rz_{\bN \setminus \rW}}\left[\cA(\rW, z_{\rW}, \rr) \not\in \supp(\cD_{\rz})\setminus(\rW, z_{\rW})\right]\right]\right].
    \end{align*}
    \endgroup
    In the first step above, Reverse Fatou's Lemma applies since each $\p_{\rS\sim \cD_{\rz}^n\hspace{-0.2em},\hspace{0.1em} \rr}[\cA(\rS, \rr) \not\in \supp(\cD_{\rz})\setminus\rS]$ is at most 1. In the last line, we separated the sampling of $\rz$ into two parts: $\rz_\rW$ (which corresponds to sampling $\rz$ at indices in $\rW$) and $\rz_{\bN \setminus \rW}$ (which corresponds to sampling $\rz$ at all remaining indices); both of these parts are independent of each other. Now note that once $\rW, \rr$ and $\rz_\rW$ are realized, $\cA(\rW, z_{\rW}, \rr)$ is fixed to be some $(w',y') \in U$. If $w' \in \rW$, then $\p_{\rz_{\bN \setminus \rW}}[\cA(\rW, z_{\rW}, \rr) \not\in \supp(\cD_{\rz})\setminus(\rW, z_{\rW})]=1$. Otherwise, $w' \in \bN \setminus \rW$; in this case, $\p_{\rz_{\bN \setminus \rW}}[\cA(\rW, z_{\rW}, \rr) \not\in \supp(\cD_{\rz})\setminus(\rW, z_{\rW})] = 1-1/\lceil 1/\eps \rceil \ge 1-\eps$, since this is the chance that the value of $\rz$ at index $w'$ is realized to not be $y'$. Thus, we have that $\p_{\rz_{\bN \setminus \rW}}[\cA(\rW, z_{\rW}, \rr) \not\in \supp(\cD_{\rz})\setminus(\rW, z_{\rW})] \ge 1-\eps$ regardless, which implies that the last quantity in the display above is at least $1-\eps$. But this implies the existence of a $z \in \{1,2,\dots,\lceil 1/\eps\rceil\}^{\mathbb{N}} $ for which
    \begin{align*}
        \limsup_{n} \p_{\rS\sim \cD_{z}^n\hspace{-0.2em},\hspace{0.1em} \rr}[\cA(\rS, \rr) \not\in \supp(\cD_{z})\setminus \rS] \ge 1-\eps,
    \end{align*}
    which completes the proof.
\end{proof}

For each $i$, let
$L_i=\{u_{k_{i,1}},u_{k_{i,2}},\ldots\}$, where
$k_{i,1}<k_{i,2}<\cdots$, be the ordering induced by the
fixed enumeration of $U$.
\begin{algorithm}[H]
\caption{Agnostic Generation}\label{alg:generation}
\begin{algorithmic}[1]
    \Require A sample $\rS = (\rx_1, \dots, \rx_n)$, countable language collection $\cC=\{L_1,L_2,\dots\}$
\For{$i=1,\ldots,n$}
    \State Set $r_i \gets
    \min\left\{
        r\in\{1,\ldots,n+1\}:
        u_{k_{i,r}}\notin S
    \right\}$\Comment{Counter for the first unseen string in $ L_i $}
\EndFor
\State Let $ o \in \arg\max_{i\in\{1,\ldots,n\}} r_{i} $
\Comment{Select a language that has the largest $r_i$ value}
\State \Return $u_{k_{o,r_o}}$ \Comment{Generate the current string in the selected language}
\end{algorithmic}
\end{algorithm}
We remark that calculating $r_i$ for each $i$  can be done by checking each string in $L_i$ in order until the first one not in the sample is found, this takes at most $n$ steps for each $i$, and since we only consider the first $n$ languages, the algorithm runs in finitely many operations, assuming access to the first $n$ strings in each language.
\begin{proof}[Proof of \Cref{thm:agnostic-generation-exponential-rate}]
Consider the execution of \Cref{alg:generation}, and write each language as
\begin{align*}
    L_i=\{u_{k_{i,1}},u_{k_{i,2}},\ldots\},
    \qquad
    k_{i,1}<k_{i,2}<\cdots .
\end{align*}
Let $M=i(\cC,\cD)$. We first notice that any language $ L_{i}\in \cC$ that satisfies $ L_i\not\subseteq \supp(\cD) $ must have $ r_{i} \leq M $. This is because, by the definition of $ i(\cC,\cD) $ in \eqref{eqn:i-C-D}, there exists $j\le M$ such that $u_{k_{i,j}}\notin \supp(\cD)$. By the definition of $r_i$, every element $u_{k_{i,r}}$
with $r<r_i$ occurs in the sample. Since $u_{k_{i,j}}$
does not occur in the sample almost surely, we have
$r_i\leq j\leq M$. Consequently, if the algorithm selects a language $ L_{o} $ with $ r_{o} > M $, it must hold that $ L_{o}\subseteq \supp(\cD) $. In this case, $u_{k_{o,r_o}}$ is guaranteed to be in $\supp(\cD) \setminus \rS$.

So, let $ L_m$ be any language in $\cC$ satisfying $ L_m\subseteq \supp(\cD) $ (which exists by assumption). Since the algorithm considers the first $n$ languages, we may restrict attention to $n\ge m$. Let
\begin{align*}
    I_{\star}=\{u_{k_{m,1}},\ldots,u_{k_{m,M}}\}
\end{align*}
be the first $M$ strings in $L_m$, and let $ p_{\star}=\min_{u\in I_{\star}}\p_{\rx\sim \cD}[\rx=u]>0 $ be the probability of the least likely string in this finite set. We then have that
\begin{align*}
    &\p_{\rS\sim \cD^{n}}[\forall u\in I_{\star} : u\in \rS ] =
    1-\p_{\rS\sim \cD^{n}}[\exists u\in I_{\star} : u\not\in \rS ]
    \\
    &\geq
    1-\sum_{u\in I_{\star}}\p_{\rS\sim \cD^{n}}[u\not\in \rS ] =
    1-\sum_{u\in I_{\star}}(1-\p_{\rx\sim \cD}[\rx=u])^{n} \\
    &\geq
    1-|I_{\star}|(1-p_{\star})^{n},
\end{align*}
where the first inequality follows by the union bound, and the last inequality follows by  the definition of $p_\star$. Thus, for all $n\ge m$, with probability at least $1-|I_{\star}|(1-p_{\star})^{n}$, the counter for $L_m$ satisfies $r_m > M$. On this event, the algorithm selects a language $L_o$ with $r_o > M$, and hence $u_{k_{o,r_o}}\in\supp(\cD)\setminus\rS$. The theorem follows by setting the constants $ c=|I_{\star}| $ and $ C=-\ln(1-p_{\star}) $.
\end{proof}

To give the proof of \Cref{thm:genlearnabilitycharacterization} we need the following lemma.

\begin{lemma}\label{lem:baddistributiongeneration}
For $ k\geq 2 $, let
$
    p_{k}=2^{-2^{k-1}}-2^{-2^{k}},
$
and let $ p_{1}=\frac{3}{4} $. Let $ \cD $ be the probability distribution on $ \mathbb{N} $ defined by $ \cD(\{k\})=p_{k} $. For $ n\geq 2 $, define
$
    f(n)=\lceil \frac{\ln(2n^{2})}{p_{n}}\rceil .
$
Then, for all sufficiently large $ n $, if $ X_{1},X_{2},\ldots $ are i.i.d. according to $ \cD $ and
$
    S_{m}=\{X_{1},\ldots,X_{m}\},
$
then
\begin{align*}
    \mathbb{P}\left(S_{f(n)}=\{1,\ldots,n\}\right)\geq 1-\frac{1}{n}.
\end{align*}
\end{lemma}

\begin{proof}
For $ k\geq 2 $, define
$
    p_{k}=2^{-2^{k-1}}-2^{-2^{k}},
$
and put the remaining mass on $ 1 $:
$
    p_{1}=1-\sum_{k=2}^{\infty}p_{k}=\frac{3}{4},
$
where we have used that
\begin{align}
    \sum_{k=2}^{m}p_{k}
    =\sum_{k=2}^{m}\left(2^{-2^{k-1}}-2^{-2^{k}}\right)
    =2^{-2}-2^{-2^{m}}.
    \tag{by telescoping}
\end{align}
so $ m\to\infty $ gives $ \sum_{k=2}^{\infty}p_{k}=2^{-2}=\frac{1}{4} $, and therefore $ p_{1}=1-\frac{1}{4}=\frac{3}{4} $.
Thus $ (p_{k})_{k\geq 1} $ is a probability distribution on $ \mathbb{N} $. Let $ \cD $ be the distribution given by $ \cD(\{k\})=p_{k} $.
For later use, compute the upper tail. For every $ n\geq 2 $,
\begin{align}
    R_{n}
    \coloneq \sum_{k>n}p_{k}
    =\sum_{k=n+1}^{\infty}
    \left(2^{-2^{k-1}}-2^{-2^{k}}\right)     \tag{by telescoping}
    =2^{-2^{n}}
    \\
\label{eq:baddist-tail}
\end{align}
Here $ R_{n}=\mathbb{P}(X_{1}>n) $ is the probability that a single draw lands above $ n $, i.e. outside $ \{1,\ldots,n\} $ on the upper side.
Define
$
    f(n)=\lceil \frac{\ln(2n^{2})}{p_{n}}\rceil .
$
By the defining property of the ceiling function,
\begin{align}
    \frac{\ln(2n^{2})}{p_{n}}
    \leq f(n)
    \leq \frac{\ln(2n^{2})}{p_{n}}+1.
    \label{eq:baddist-ceil-bound}
\end{align}

We now bound the complement of the desired event. The event $ S_{f(n)}\neq \{1,\ldots,n\} $ can only occur if either some element of $ \{1,\ldots,n\} $ is not sampled, or some element larger than $ n $ is sampled. For a fixed $ i\leq n $, the probability that $ i $ is missing from all $ f(n) $ samples is $ (1-p_{i})^{f(n)} $. Hence, by the union bound, the probability that some $ i\leq n $ is missing is at most $ \sum_{i=1}^{n}(1-p_{i})^{f(n)} $.

Similarly, by the union bound, the probability that some sample is larger than $ n $ is at most
\begin{align}
    \mathbb{P}\left(\exists j\leq f(n): X_{j}>n\right)
    \leq \sum_{j=1}^{f(n)}\mathbb{P}(X_{j}>n)
    =f(n)R_{n}.
\label{eq:baddist-upper-outside}
\end{align}
Since $ p_{i}\geq p_{n} $ for all $ i\leq n $,
\begin{align}
    (1-p_{i})^{f(n)}
    \leq e^{-f(n)p_{i}}
    \leq e^{-f(n)p_{n}}
    \tag{by $1-x\leq e^{-x}$ and $p_i\geq p_n$}
    \\
\label{eq:baddist-missing-one}
\end{align}
Combining \Cref{eq:baddist-upper-outside} and \Cref{eq:baddist-missing-one}, we get
\begin{align}
    \mathbb{P}\left(S_{f(n)}\neq \{1,\ldots,n\}\right)
    &\leq \sum_{i=1}^{n}(1-p_{i})^{f(n)}+f(n)R_{n}
    \tag{by union bound}\\
    &\leq n e^{-f(n)p_{n}}+f(n)R_{n}
    \tag{by \Cref{eq:baddist-missing-one}}\\
    &\leq \frac{1}{2n}+f(n)R_{n}
    \tag{by \Cref{eq:baddist-ceil-bound}}
\\
\label{eq:baddist-complement-bound}
\end{align}
In the last line we used \Cref{eq:baddist-ceil-bound}, which gives $ f(n)p_{n}\geq \ln(2n^{2}) $. It remains to control the term $ f(n)R_{n} $. By \Cref{eq:baddist-tail} and \Cref{eq:baddist-ceil-bound},
\begin{align}
    f(n)R_{n}
    &=f(n)2^{-2^n}
    \tag{by \Cref{eq:baddist-tail}}\\
    &\leq \left(\frac{\ln(2n^{2})}{p_{n}}+1\right)2^{-2^{n}}
    \tag{by \Cref{eq:baddist-ceil-bound}}\\
    &\leq \left(\frac{\ln(2n^{2})}{p_{n}}+1\right)2^{-2^{n}}.
\label{eq:baddist-f-tail-bound}
\end{align}
For $ n\geq 2 $, we have $ 2^{-2^{n-1}}\leq \frac{1}{4} $. Hence
$
    p_{n}
    =2^{-2^{n-1}}(1-2^{-2^{n-1}})
    \geq \frac{3}{4}2^{-2^{n-1}}$
    by $2^{-2^{n-1}}\leq 1/4$
Equivalently,
\begin{align}
    \frac{1}{p_{n}}
    \leq \frac{4}{3}2^{2^{n-1}}.
\label{eq:baddist-pn-reciprocal}
\end{align}
Therefore,
\begin{align}
    n f(n)R_{n}
    &\leq n\left(\frac{\ln(2n^{2})}{p_{n}}+1\right)2^{-2^{n}}
    \tag{by \Cref{eq:baddist-f-tail-bound}}\\
    &= n\ln(2n^{2})\frac{1}{p_{n}}2^{-2^{n}}
    +n2^{-2^{n}}\notag\\
    &\leq \frac{4}{3}n\ln(2n^{2})2^{2^{n-1}}2^{-2^{n}}
    +n2^{-2^{n}}
    \tag{by \Cref{eq:baddist-pn-reciprocal}}\\
    &=\frac{4}{3}n\ln(2n^{2})2^{-2^{n-1}}+n2^{-2^{n}} \notag
    \\
    \Rightarrow f(n)R_{n}
    &\leq \frac{4}{3}n\ln(2n^{2})2^{-2^{n-1}}+n2^{-2^{n}}
    \label{eq:baddist-nfrn-bound}
\end{align}
The right-hand side of \Cref{eq:baddist-nfrn-bound} converges to $0$. Thus, for all sufficiently large $ n $, $ n f(n)R_{n}\leq \frac{1}{2} $, i.e.
\begin{align}
    f(n)R_{n}\leq \frac{1}{2n}.
\label{eq:baddist-tail-small}
\end{align}
Therefore, for all sufficiently large $ n $,
\begin{align*}
    \mathbb{P}\left(S_{f(n)}=\{1,\ldots,n\}\right)
    &\geq 1-\frac{1}{2n}-\frac{1}{2n} \\
    &=1-\frac{1}{n}.
\end{align*}
\end{proof}

\begin{proof}[Proof of \Cref{thm:genlearnabilitycharacterization}]

Let $ U=\{1,2\}\times \mathbb{N} $, and let
$ \{2,3,5,\ldots\}=\{q_{1},q_{2},q_{3},\ldots\} $ denote the set of prime
numbers. Let the enummeration of $U$ be as follows $u_{2m-1}=(1,m)$ and $u_{2m}=(2,m)$ for every $m\in\mathbb{N}$. For $ i\geq 1 $, define
$
    L_i=\{(1,q_i^j):j\geq 1\},
$
and define the limit language
$
    L_0=\{(1,q_i):i\geq 1\}.
$
We also add decoy languages. For $ i,m\geq 1 $, let
$
    D_{i,m}
    =
    \{(1,q_i^j):1\leq j\leq m\}
    \cup
    \{(2,q_i^j):j\geq m\},
$
and for $ m\geq 1 $, let
$
    D_{0,m}
    =
    \{(1,q_i):1\leq i\leq m\}
    \cup
    \{(2,q_i):i\geq m\}.
$
Set
\begin{align*}
    \cC
    =
    \{L_0,L_1,L_2,\ldots\}
    \cup
    \{D_{i,m}:i\geq 0,\ m\geq 1\}.
\end{align*}
This class is countable.
For any displayed language $A\in\cC$, write
$A_{[r]}$ for the set of the first $r$ elements of $A$ in that ordering given by the enumaration of $U$.
Thus, for example, $(D_{0,m})_{[m]}$ is $(1,q_1),\ldots,(1,q_m)$.

We first prove a deterministic sequence lower bound for this class. Let
$ h:\mathbb{N}\to\mathbb{N} $, and let $ \cA $ be any algorithm
which, on input $ n $ elements of $U$, outputs $h(n)$ elements of $U$. For a
sequence of elements $x_1,\ldots$ from $U$ with support $S=\{x_i\}_{i=1}^{\infty}$, write
\begin{align*}
 \generravg(\cA,n)
 \coloneq
 \frac{1}{h(n)}\sum_{\ell=1}^{h(n)}
 \ind\{\cA(x_{1},\ldots,x_{n})_{\ell}\not\in S\setminus \{x_{1},\ldots,x_{n}\}\}.
\end{align*}
We only
consider sequences $x_1,x_2,\ldots$ generated by the following phases, and all
strings presented in these phases have the form $(1,q_i^j)$. For a given algorithm $\cA$ and a given function $h$, the phases are defined as follows.
In phase $i\geq 1$, start presenting
\begin{align*}
    (1,q_i),(1,q_i^2),(1,q_i^3),\ldots.
\end{align*}
Suppose that after the prefix has length $n$, the largest exponent presented
in the current phase is $j$. If at least $\lceil h(n)/2\rceil$ output entries
of $\cA(x_1,\ldots,x_n)$ are of the form $(1,q_i^k)$ with $k>j$, then end
phase $i$ and move to phase $i+1$. If this never happens in phase $i$, remain
in phase $i$ forever and keep presenting the powers of $q_i$.

Let $x_1,x_2,\ldots$ be the resulting infinite sequence and let
$S=\{x_t:t\in\mathbb{N}\}$. If the phases continue forever, then
$L_0\subseteq S$, no element of the form $(2,q)$ belongs to $S$, and for every
$m\geq 1$ the decoy $D_{0,m}$ satisfies
$
    (D_{0,m})_{[m]}\subseteq S$ and $
    D_{0,m}\nsubseteq S.
$
Hence every distribution with support $S$ has $i(\cD,\cC)=\infty$. Moreover,
at every phase transition, at least half of the outputs are future powers from
the phase just abandoned, and those elements are not in the final support $S$.
Thus $\generravg(\cA,n)\geq 1/2$ at infinitely many transition times.

On the other hand, if the construction stalls in phase $i$, then $L_i\subseteq
S$, no element of the form $(2,q)$ belongs to $S$, and for every $m\geq 1$ the
decoy $D_{i,m}$ satisfies
$
    (D_{i,m})_{[m]}\subseteq S$ and $
    D_{i,m}\nsubseteq S.
$
Again every distribution with support $S$ has $i(\cD,\cC)=\infty$. Since the
phase never triggers, at infinitely many times fewer than half of the output
entries are future elements of the current column, which are the only possible
fresh elements of $S$. Hence $\generravg(\cA,n)\geq 1/2$ for infinitely many
$n$.

Thus, for every $h$ and every deterministic algorithm $\cA$ outputting
$h(n)$ elements, the phase construction produces a sequence
$x_1,x_2,\ldots$ with support $S$ such that some main language $L_i$,
$i\geq 0$, is contained in $S$, every distribution with support $S$ lies in
$\bar{\cP}_{gen}(\cC)$, and for infinitely many $n$,
\begin{align}
    \generravg(\cA,n)\geq \frac{1}{2}.
\label{eq:mainproof-deterministic-lower-bound}
\end{align}
Assume for a contradiction that there exists deterministic algorithm $ \bar{\cA} $ such that for any $ \cD \in \bar{\cP}_{gen}(\cC) $ it holds that
$
    \limsup_{n\rightarrow \infty}\generr(\bar{\cA},\cD,n) < 1/2.
$
Given $ \bar{\cA} $, we now construct an algorithm $ \cA $ that works on a deterministic sequences $ x_{1},x_{2},\ldots $ created as above and obtains $ \generravg $ strictly less than $ 1/2 $, which is a contradiction with what we showed previously, thus the assumption of the existence of $ \bar{\cA} $ is not true, implying the claim of the theorem. Extend $f$(from \Cref{lem:baddistributiongeneration}) to $n=1$  by $ 1 $  and let
$
    g(n)=2^{n}f(n).
$
For $ n\geq 1 $ and $ \sigma\in\{-1,1\}^{g(n)} $, split $ \sigma $ into $ f(n) $ consecutive blocks of length $ 2^{n} $. For each $ i\in\{1,\ldots,f(n)\} $, define
$
    T_i^{(n)}
    =
    \min\left\{t\in\{1,\ldots,2^n\}:
    \sigma_{(i-1)2^n+t}=1\right\},
$
with the convention that $ T_i^{(n)}=\infty $ if the set above is empty. We then define the resampling index $ C_i^{(n)} $ by
\begin{align*}
 C^{(n)}_{i}=
\begin{cases}
1, & T_i^{(n)}\le 2,\\
k, & 2^{k-1}<T_i^{(n)}\le 2^k,\quad 2\le k\le n,\\
*, & T_i^{(n)}=\infty .
\end{cases}
\end{align*}
If we let $ \sigma $ be a random string, then $ C^{(n)}_{1},\ldots,C^{(n)}_{f(n)} $ are independent, and for every $ i\in\{1,\ldots,f(n)\} $,
\begin{align}
    \mathbb{P}\left(C_i^{(n)}=1\right)=\frac{3}{4},\qquad
    \mathbb{P}\left(C_i^{(n)}=k\right)=2^{-2^{k-1}}-2^{-2^k}
    \quad (2\le k\le n),\nonumber
    \\
\mathbb{P}\left(C_i^{(n)}=*\right)=2^{-2^n}. \label{eq:mainproof-C-star-probability}
\end{align}
For convenience, define $ \cA' $ as follows. On the event that $ C_i^{(n)}\neq * $ for every $ i\in\{1,\ldots,f(n)\} $, set
$
    \cA'(\sigma,x_{1},\ldots,x_{n})
    =
    \bar{\cA}(x_{C_1^{(n)}},\ldots,x_{C_{f(n)}^{(n)}}),
$
and define $ \cA' $ arbitrarily on the complementary event.
Finally, define the deterministic algorithm $ \cA $ by enumerating all possible choices of the random string. For each $ n $, fix an enumeration
$
    \{-1,1\}^{g(n)}=\{\sigma^{(1,n)},\ldots,\sigma^{(2^{g(n)},n)}\}.
$
On input $ (x_{1},\ldots,x_{n}) $, the algorithm $ \cA $ outputs $ 2^{g(n)} $ entries, given by
$
    \cA(x_{1},\ldots,x_{n})_{m}
    =
    \cA'(\sigma^{(m,n)},x_{1},\ldots,x_{n}),
    1\leq m\leq 2^{g(n)}.
$
Apply the phase construction above to this deterministic algorithm $\cA$ and
the function $h(n)=2^{g(n)}$. Let $x_1,x_2,\ldots$ be the resulting sequence
and write $S=\{x_i:i\in\mathbb{N}\}$. Define the following distribution $\cD$
over $U$: first sample $\tilde{C}$ according to the distribution from
\Cref{lem:baddistributiongeneration}, and then output $x_{\tilde{C}}$. Notice
that by construction $\supp(\cD)=S$. By the phase construction, some main
language $L_i$, $i\geq 0$, is contained in $S$, and $i(\cD,\cC)=\infty$.
Therefore $\cD\in \bar{\cP}_{gen}(\cC)$, so the assumed guarantee for
$\bar{\cA}$ gives
$
    \limsup_{n\rightarrow \infty}\generr(\bar{\cA},\cD,n)<1/2.
$
We now notice that
\begin{align}
 \p\left[\exists i \in\{1,\ldots,f(n)\}: C_i^{(n)}=*\right]
 &\leq \sum_{i=1}^{f(n)}\p\left[C_i^{(n)}=*\right]
 \tag{by union bound}\\
 &= f(n)2^{-2^n}.
 \tag{by \Cref{eq:mainproof-C-star-probability}}
 \\
\label{eq:mainproof-star-union-bound}
\end{align}
implying that with probability at least $ 1-f(n)2^{-2^n} $ it holds that $ C_i^{(n)}\in[n] $ for all $ i\in\{1,\ldots,f(n)\} $. Furthermore, for $ \tilde{C}_{1},\ldots,\tilde{C}_{f(n)} $ i.i.d. distributed as $ \tilde{C} $ we have that
\begin{align}
    \p\left[\forall i \in \{1,\ldots,f(n)\}: \tilde{C}_{i}\leq n\right]
    &=1-\p\left[\exists i \in \{1,\ldots,f(n)\}: \tilde{C}_{i}>n\right] \notag\\
    &\geq 1-\sum_{i=1}^{f(n)}\p[\tilde{C}_{i}>n]
    \tag{by union bound}\\
    &=1-f(n)2^{-2^n}
    \tag{by \Cref{eq:baddist-tail}}\\
    \Rightarrow \p\left[\forall i \in \{1,\ldots,f(n)\}: \tilde{C}_{i}\leq n\right]
    &\geq 1-f(n)2^{-2^n}.
\label{eq:mainproof-tilde-truncation-bound}
\end{align}
With this notation, define the good truncation events
$
    E_n
    =
    \{\forall i\in\{1,\ldots,f(n)\}: C_i^{(n)}\neq *\}$ and
    $
    \widetilde{E}_n
    =
    \{\forall i\in\{1,\ldots,f(n)\}: \tilde{C}_i\leq n\},
$
so that, by \Cref{eq:mainproof-star-union-bound} and \Cref{eq:mainproof-tilde-truncation-bound},
\begin{align}
    \p[E_n]\geq 1-f(n)2^{-2^n},
    \qquad
    \p[\widetilde{E}_n^c]\leq f(n)2^{-2^n}.
\label{eq:mainproof-good-event-probabilities}
\end{align}
Moreover,
\begin{align}
    (C_1^{(n)},\ldots,C_{f(n)}^{(n)})\mid E_n
    \stackrel{d}{=}
    (\tilde{C}_{1},\ldots,\tilde{C}_{f(n)})\mid \widetilde{E}_n .
\label{eq:mainproof-conditional-law}
\end{align}
Therefore, using the definition of $ \cA' $,
\begin{align}
&\p\left[
    \cA'(\sigma,x_1,\ldots,x_n)
    \in \supp(\cD)\setminus\{x_1,\ldots,x_n\}
\right]
\notag\\
&\quad\geq
\p\left[
    \bar{\cA}\left(x_{C_1^{(n)}},\ldots,x_{C_{f(n)}^{(n)}}\right)
    \in \supp(\cD)\setminus\{x_1,\ldots,x_n\}
    \,\middle|\, E_n
\right]\p[E_n]
\tag{by conditioning on $E_n$}\\
&\quad=
\p\left[
    \bar{\cA}\left(x_{\tilde{C}_{1}},\ldots,x_{\tilde{C}_{f(n)}}\right)
    \in \supp(\cD)\setminus\{x_1,\ldots,x_n\}
    \,\middle|\, \widetilde{E}_n
\right]\p[E_n]
\tag{by \Cref{eq:mainproof-conditional-law}}\\
&\quad\geq
\p\left[
    \bar{\cA}\left(x_{\tilde{C}_{1}},\ldots,x_{\tilde{C}_{f(n)}}\right)
    \in \supp(\cD)\setminus\{x_1,\ldots,x_n\},
    \widetilde{E}_n
\right]\left(1-f(n)2^{-2^n}\right)
\tag{by \Cref{eq:mainproof-good-event-probabilities}}\\
&\quad\geq
\left(
\p\left[
    \bar{\cA}\left(x_{\tilde{C}_{1}},\ldots,x_{\tilde{C}_{f(n)}}\right)
    \in \supp(\cD)\setminus\{x_1,\ldots,x_n\}
\right]
-f(n)2^{-2^n}
\right)\left(1-f(n)2^{-2^n}\right).
\tag{by $P[A\cap B]\geq P[A]-P[B^{c}]$}\\
\Rightarrow &\p\left[
    \cA'(\sigma,x_1,\ldots,x_n)
    \in \supp(\cD)\setminus\{x_1,\ldots,x_n\}
\right]\nonumber
\\
&\quad\geq
\left(
\p\left[
    \bar{\cA}\left(x_{\tilde{C}_{1}},\ldots,x_{\tilde{C}_{f(n)}}\right)
    \in \supp(\cD)\setminus\{x_1,\ldots,x_n\}
\right]
-f(n)2^{-2^n}
\right)\left(1-f(n)2^{-2^n}\right).
\label{eq:mainproof-Aprimesuccess-truncation}
\end{align}
Now define the coverage event
$
    H_n=\{\{\tilde{C}_{1},\ldots,\tilde{C}_{f(n)}\}=\{1,\ldots,n\}\}.
$
By \Cref{lem:baddistributiongeneration},
\begin{align}
    \p[H_n]\geq 1-\frac{1}{n}.
\label{eq:mainproof-coverage-probability}
\end{align}
On $ H_n $ we have $
    \{x_{\tilde{C}_{1}},\ldots,x_{\tilde{C}_{f(n)}}\}
    =
    \{x_1,\ldots,x_n\}.
$
Hence,
\begin{align}
&\p\left[
    \bar{\cA}\left(x_{\tilde{C}_{1}},\ldots,x_{\tilde{C}_{f(n)}}\right)
    \in \supp(\cD)\setminus\{x_1,\ldots,x_n\}
\right]
\notag\\
&\quad\geq
\p\left[
    \bar{\cA}\left(x_{\tilde{C}_{1}},\ldots,x_{\tilde{C}_{f(n)}}\right)
    \in \supp(\cD)\setminus\{x_1,\ldots,x_n\},\{x_{\tilde{C}_{1}},\ldots,x_{\tilde{C}_{f(n)}}\}
    =
    \{x_1,\ldots,x_n\}\right]
\tag{by intersecting with $H_n$}\\
&\quad\geq
\p\left[
    \bar{\cA}\left(x_{\tilde{C}_{1}},\ldots,x_{\tilde{C}_{f(n)}}\right)
    \in \supp(\cD)\setminus\{x_{\tilde{C}_{1}},\ldots,x_{\tilde{C}_{f(n)}}\}
\right]
-\frac{1}{n}
\tag{by  $P[A\cap B]\geq P[A]-P[B^{c}]$ and \Cref{eq:mainproof-coverage-probability}}\\
&\quad=1-\generr(\bar{\cA},\cD,f(n))-\frac{1}{n}.
\tag{by definition of $\generr$}\\
\Rightarrow
&\p\left[
    \bar{\cA}\left(x_{\tilde{C}_{1}},\ldots,x_{\tilde{C}_{f(n)}}\right)
    \in \supp(\cD)\setminus\{x_1,\ldots,x_n\}
\right]
\notag\\
&\quad\geq
1-\generr(\bar{\cA},\cD,f(n))-\frac{1}{n}.
\label{eq:mainproof-coverage-success-bound}
\end{align}
Combining \Cref{eq:mainproof-Aprimesuccess-truncation} and \Cref{eq:mainproof-coverage-success-bound} gives
\begin{align}
&\p\left[
    \cA'(\sigma,x_1,\ldots,x_n)
    \in \supp(\cD)\setminus\{x_1,\ldots,x_n\}
\right]
\notag\\
&\quad\geq
\left(1-\generr(\bar{\cA},\cD,f(n))-\frac{1}{n}-f(n)2^{-2^n}\right)
\left(1-f(n)2^{-2^n}\right).
\label{eq:mainproof-Aprimesuccess-combined}
\end{align}
Finally, by the definition of the deterministic algorithm $ \cA $, averaging over its $2^{g(n)}$ outputs is the same as averaging $ \cA' $ over a uniform random string $ \sigma\in\{-1,1\}^{g(n)} $. Hence,
\begin{align}
&\frac{1}{2^{g(n)}}\sum_{m=1}^{2^{g(n)}}
\ind\left\{
    \cA(x_1,\ldots,x_n)_m
    \in S\setminus\{x_1,\ldots,x_n\}
\right\}\label{eq:mainproof-average-fresh-lower-bound}
\\
&=
\p_{\sigma}\left[
    \cA'(\sigma,x_1,\ldots,x_n)
    \in S\setminus\{x_1,\ldots,x_n\}
\right]
\tag{by enumeration of $\{-1,1\}^{g(n)}$}\\
&\geq
\left(1-\generr(\bar{\cA},\cD,f(n))-\frac{1}{n}-f(n)2^{-2^n}\right)
\left(1-f(n)2^{-2^n}\right).
\tag{by \Cref{eq:mainproof-Aprimesuccess-combined}}
\end{align}
Therefore, \Cref{eq:mainproof-average-fresh-lower-bound} implies that
\begin{align}
  &\generravg(\cA,n)
  =1- \frac{1}{2^{g(n)}}\sum_{m=1}^{2^{g(n)}}
\ind\left\{
    \cA(x_1,\ldots,x_n)_m
    \in S\setminus\{x_1,\ldots,x_n\}
\right\}
\notag\\
&\leq
1-\left(1-\generr(\bar{\cA},\cD,f(n))-\frac{1}{n}-f(n)2^{-2^n}\right)
\left(1-f(n)2^{-2^n}\right)
\tag{by \Cref{eq:mainproof-average-fresh-lower-bound}}\\
&=
\generr(\bar{\cA},\cD,f(n))
+\frac{1}{n}
+2f(n)2^{-2^n}
\negmedspace-f(n)2^{-2^n}\negmedspace\left(\generr(\bar{\cA},\cD,f(n))+\frac{1}{n}+f(n)2^{-2^n}\right)
\tag{by expanding}\\
&\leq
\generr(\bar{\cA},\cD,f(n))
+\frac{1}{n}
+2f(n)2^{-2^n}.
\tag{dropping a nonnegative term}\\
&\Rightarrow
\generravg(\cA,n)
\leq
\generr(\bar{\cA},\cD,f(n))
+\frac{1}{n}
+2f(n)2^{-2^n}.
\label{eq:mainproof-generravg-upper-bound}
\end{align}
We notice that if $ \limsup_{n\rightarrow \infty}\generr(\bar{\cA},\cD,n)<1/2 $ then $ \limsup_{n\rightarrow \infty}\generr(\bar{\cA},\cD,f(n))<1/2 $ as well, since $ f(n)\to\infty $. Moreover, $ \frac{1}{n}+2f(n)2^{-2^n}\to 0 $. Therefore, by \Cref{eq:mainproof-generravg-upper-bound},
$
    \limsup_{n\rightarrow \infty}\generravg(\cA,n)<1/2.
$
which contradicts \Cref{eq:mainproof-deterministic-lower-bound}. Thus, the assumption of the existence of $ \bar{\cA} $ is not true, implying the claim of the theorem.
\end{proof}

\begin{proof}[Proof of \Cref{thm:generation-exponential-rate-lower-bound}]
    Let $ m=|L\cap L'| $. We first consider the case that $ m=0 $. In this case, let $s$ be a string which is in neither of $ L $ and $ L' $; such an $s$ exists by assumption that $U \setminus (L\cup L')\not=\emptyset$. Now, let $ \cD_{0} $ be a distribution that has mass $1/2 $ on $ s $, and spreads the rest of its mass arbitrarily on the strings in $L$, ensuring positive mass on every string. Similarly, let $ \cD_{1} $ be a distribution that has mass $1/2 $ on $ s $, and spreads the rest of its mass arbitrarily on the strings in $L'$, again ensuring that every string gets positive mass. Note that $L \subseteq \supp(\cD_0)$ and $L' \subseteq \supp(\cD_1)$. Now consider realizing the sample $\rS=s^n$, which gets realized with the same probability in both $\cD_0$ and $\cD_1$; in this case, observe that $\supp(\cD_0) \setminus s^n = L$ and $\supp(\cD_1) \setminus s^n = L'$. Then, for any realization $\rr$ of $\cA$'s randomness, by the disjointedness of $L$ and $L'$, it must hold that either $\cA(s^n ,\rr) \notin \supp(\cD_0) \setminus s^n$ or $\cA(s^n ,\rr) \notin \supp(\cD_1) \setminus s^n$. Namely, we have that $\ind\{\cA(s^n,\rr) \not\in \supp(\cD_{0})\setminus s^n\}+\ind\{\cA(s^n,\rr) \not\in \supp(\cD_{1}) \setminus s^n\} \geq 1 $. Therefore, for any $n$,
    \begin{align*}
        &\max_{b\in\{0,1\}} \p_{\rS\sim \cD_{b}^{n},\rr}[\cA(\rS, \rr) \not\in \supp(\cD_{b}) \setminus \rS]
        \\
        \ge& \frac{1}{2}\left( \p_{\rS\sim \cD_{0}^{n},\rr}[\cA(\rS, \rr) \not\in \supp(\cD_{0}) \setminus \rS]+\p_{\rS\sim \cD_{1}^{n},\rr}[\cA(\rS, \rr) \not\in \supp(\cD_{1}) \setminus \rS]\right) \\
        =& \frac12 \left(\e_{\rS \sim \cD_0^n}\e_{\rr}\left[\ind\{\cA(\rS,\rr) \not\in \supp(\cD_{0})\setminus \rS\}\right]+\e_{\rS \sim \cD_1^n}\e_{\rr}\left[\ind\{\cA(\rS,\rr) \not\in \supp(\cD_{1}) \setminus \rS\}\right]\right) \\
        \ge& \frac12 \left(\p_{\rS \sim \cD_0^n}[\rS = s^n] \cdot \e_{\rr}\left[\ind\{\cA(s^n,\rr) \not\in \supp(\cD_{0})\setminus s^n\}\right]\right.
        \\
        &+\left.\p_{\rS \sim \cD_1^n}[\rS = s^n]\cdot \e_{\rr}\left[\ind\{\cA(s^n,\rr) \not\in \supp(\cD_{1}) \setminus s^n\}\right]\right) \\
        =& \frac{1}{2^{n+1}} \left(\e_{\rr}\left[\ind\{\cA(s^n,\rr) \not\in \supp(\cD_{0})\setminus s^n\}+\ind\{\cA(s^n,\rr) \not\in \supp(\cD_{1}) \setminus s^n\}\right]\right)
        \\
        &\ge \frac{1}{2^{n+1}} \ge \exp(-2n).
    \end{align*}
    The above implies that $\exists b\in\{0,1\}$ such that $\p_{\rS\sim \cD_{b}^{n},\rr}[\cA(\rS, \rr) \not\in \supp(\cD_{b}) \setminus \rS] \ge \exp(-2n)$ for infinitely many $n$.

    Next, consider the case that $ m > 0 $. Fix $\eps = \frac{1}{16m}$. Let $\cD_0$ be a distribution that has mass $(1-\eps)/m$ on every string in $L \cap L'$, and spreads the rest of the $\eps$ mass arbitrarily on $L \setminus L'$, ensuring that every string gets positive mass. Similarly, let $\cD_1$ be a distribution that has mass $(1-\eps)/m$ on every string in $L \cap L'$, and spreads the rest of the $\eps$ mass arbitrarily on $L' \setminus L$, again ensuring that every string gets positive mass. Note that $\supp(\cD_0)=L$ and $\supp(\cD_1)=L'$. Let $S$ be any realization of the sample $\rS$ such that the set of distinct strings in $S$ is exactly $L \cap L'$, i.e., $\{s: s\in S\}=L \cap L'$. We observe that the probability of realizing such an $S$ is the same under both $\cD_0$ and $\cD_1$ by definition of the distributions. Furthermore, we have that $\supp(\cD_0) \setminus S = L \setminus L'$ and $\supp(\cD_1) \setminus S = L' \setminus L$. This means that for any realization of the randomness $\rr$, it must hold that either $\cA(S, \rr) \notin \supp(\cD_0) \setminus S$ or $\cA(S, \rr) \notin \supp(\cD_1) \setminus S$, implying that $\ind\{\cA(S, \rr) \notin \supp(\cD_0) \setminus S\} + \ind\{\cA(S, \rr) \notin \supp(\cD_1) \setminus S\} \ge 1$. Therefore, for any $n$,
    \begingroup
    \allowdisplaybreaks
    \begin{align*}
        &\max_{b\in\{0,1\}} \p_{\rS\sim \cD_{b}^{n},\rr}[\cA(\rS, \rr) \not\in \supp(\cD_{b}) \setminus \rS]
        \\
        &\ge \frac{1}{2}\left( \p_{\rS\sim \cD_{0}^{n},\rr}[\cA(\rS, \rr) \not\in \supp(\cD_{0}) \setminus \rS]+\p_{\rS\sim \cD_{1}^{n},\rr}[\cA(\rS, \rr) \not\in \supp(\cD_{1}) \setminus \rS]\right) \\
        =& \frac12 \left(\e_{\rS \sim \cD_0^n}\e_{\rr}\left[\ind\{\cA(\rS,\rr) \not\in \supp(\cD_{0})\setminus \rS\}\right]+\e_{\rS \sim \cD_1^n}\e_{\rr}\left[\ind\{\cA(\rS,\rr) \not\in \supp(\cD_{1}) \setminus \rS\}\right]\right) \\
        \ge& \frac12 \Big(\sum_{\stackrel{S\in (L\cap L')^{n}}{\{s:s\in S\}=L\cap L'}} \left(\p_{\rS \sim \cD_0^n}[\rS=S] \cdot \e_{\rr}\left[\ind\{\cA(S,\rr) \not\in \supp(\cD_{0})\setminus S\}\right]\right.
        \\
        &\left.\left.
        +\p_{\rS \sim \cD_1^n}[\rS = S]\cdot \e_{\rr}\left[\ind\{\cA(S,\rr) \not\in \supp(\cD_{1}) \setminus S\}\right]\right)\right) \\
        =& \frac12 \Big(\sum_{\stackrel{S\in (L\cap L')^{n}}{\{s:s\in S\}=L\cap L'}} \p_{\rS \sim \cD_0^n}[\rS=S] \cdot \left(\e_{\rr}\left[\ind\{\cA(S,\rr) \not\in \supp(\cD_{0})\setminus S\}\right]
        \right.\\
        &+\left. \e_{\rr}\left[\ind\{\cA(S,\rr) \not\in \supp(\cD_{1}) \setminus S\}\right]\right)\Big) \\
        \ge& \frac12 \sum_{\stackrel{S\in (L\cap L')^{n}}{\{s:s\in S\}=L\cap L'}} \p_{\rS \sim \cD_0^n}[\rS=S] = \frac{1}{2} \cdot \p_{\rS \sim \cD_0^n}\left[\{s\in \rS\}=L \cap L'\right].
    \end{align*}
    \endgroup
    We will now show that $\p_{\rS \sim \cD_0^n}\left[\{s\in \rS\}=L \cap L'\right]$ is at least $\exp(-2n)/2$ for all large enough $n$; this will imply that $\exists b \in \{0,1\}$ such that $\p_{\rS\sim \cD_{b}^{n},\rr}[\cA(\rS, \rr) \not\in \supp(\cD_{b}) \setminus \rS] \ge \exp(-2n)/4$ for infinitely many $n$, completing the proof.

    So now, notice that
    \begin{align*}
        \p_{\rS \sim \cD_0^n}[\exists s \in L \cap L': s \notin \rS] &\le \sum_{s \in L \cap L'} \p_{\rS \sim \cD_0^n}[s \notin \rS] = m\left(1-\frac{1-\eps}{m}\right)^n
        \le m\exp(-n(1-\eps)/m) \\
        &\le m\exp(-n/2m) = m\exp(-n/4m)\exp(-n/4m) \\
        &\le \frac{4m^2}{n}\cdot\exp(-n/4m),
    \end{align*}
    where in the third-to-last inequality, we used that $1+x \le e^x$ for all real $x$, in the second-to-last inequality, we used that $\eps \le 1/2$, and in the last inequality, we used that $e^{-x} \le 1/x$ for all $x > 0$.

    Similarly, we have that
    \begin{align*}
        \p_{\rS \sim \cD_0^n}[\exists s \in \rS: s \notin L \cap L'] &= 1-\p_{\rS \sim \cD_0^n}[\forall s \in \rS: s \in L \cap L'] = 1-(1-\eps)^n
        \\
        &\le 1-\exp(-2n\eps) = 1-\exp(-n/8m),
    \end{align*}
    where in the inequality above, we used that $1-x \ge \exp(-2x)$ for $x \in [0,1/2]$.

    Combining the above, we obtain that
    \begin{align*}
        \p_{\rS \sim \cD_0^n}[\{s\in \rS\}=L \cap L'] &= 1-\p_{\rS \sim \cD_0^n}\left[\exists s \in L \cap L': s \notin \rS \text{ or } \exists s \in \rS: s \notin L \cap L'\right] \\
        &\ge 1- \p_{\rS \sim \cD_0^n}[\exists s \in L \cap L': s \notin \rS] - \p_{\rS \sim \cD_0^n}[\exists s \in \rS: s \notin L \cap L'] \tag{union bound} \\
        &\ge \exp(-n/8m) - \frac{4m^2}{n}\cdot\exp(-n/4m)
        \\
        &= \exp(-n/8m)\left(1- \frac{4m^2}{n}\cdot\exp(-n/8m)\right).
    \end{align*}
    For $n \ge 8m^2$, the last quantity above is at least $\exp(-n/8m)/2$, which, since $m \ge 1$, is at least $\exp(-2n)/2$.
\end{proof}

\section{Other objectives for agnostic generation}
\label{sec:appendix-agnostic-generation-other-benchmarks}

A different model to consider for generation is one where the collection $\cC$ comprises of \textit{distributions}, instead of languages. That is, consider a reference collection $\cC=\{D_1,D_2,\dots\}$, where each $D_i$ is a distribution over the universe $U$. In this case, a natural objective to consider would be the following: given $n$ samples drawn i.i.d. from the unknown distribution $\cD$, an algorithm $\cA$, using randomness $\rr$, outputs a distribution $\cA(\rS, \rr)$ over strings so as to minimize:
\begin{align}
    \label{eqn:alternative-benchmark-1}
    \e_{\rS\sim \cD^{n}, \rr}\left[\p_{\rx\sim \cA(\rS, \rr)\backslash \rS}[\rx \not\in \supp(\cD)] - \inf_{i \in \bN}\p_{\rx\sim D_{i}\backslash \rS}[\rx \not\in \supp(\cD)]\right].
\end{align}
Here, $ D\backslash \rS $ denotes the distribution $ D $ conditioned on not producing any string in $ \rS $ (we assume the distributions has infinite support such that $ D\backslash \rS $ is well-defined). Namely, the algorithm tries to output a generating distribution, which, after renormalizing outside the input sample $\rS$, is competitive with the best distribution from the collection. However, we can clearly see how $\inf_{i \in \bN}\p_{\rx\sim D_{i}\backslash \rS}[\rx \not\in \supp(\cD)]$ depends heavily on the realized input sample $\rS$. That is, the benchmark distribution changes drastically with the input data, which makes it difficult to reason about what exactly an algorithm should output. Indeed, this departs significantly from standard agnostic learning settings, where the benchmark to compete with generally only depends on the distribution of the input, and not on the realized sample itself.

One may then consider the following objective which is easier than \eqref{eqn:alternative-benchmark-1}:
\begin{align}
    \label{eqn:alternative-benchmark-2}
    \e_{\rS\sim \cD^{n}, \rr}\left[\p_{\rx\sim \cA(\rS, \rr)\backslash \rS}[\rx \not\in \supp(\cD)]\right] - \inf_{i \in \bN}\e_{\rS \sim \cD^n}\left[\p_{\rx\sim D_{i}\backslash \rS}[\rx \not\in \supp(\cD)]\right].
\end{align}
While this decouples the benchmark distribution from the realized sample, the benchmark distribution still depends on the size $n$ of the sample.

\paragraph{Example: best distribution changes with $ n $.}

Consider the universe to be $ U=\mathbb{Z}$, and the collection of distributions $\cC=\{D_{2},\ldots \}$, where $ D_{i} $ puts mass $1/i$ on the integer $ i $, and spreads the rest of its mass on $ \{\ldots,-2,-1,0\} $. Let the target distribution $ \cD $ have mass $ 1/2^{i} $ on each $ i=2,\ldots $. Note that support of $\cD_i$ and $\cD$ only have the integer $i$ in common.
Then, for any sample size $ n $, we have that $ \e_{\rS\sim \cD^{n}}[\p_{\rx\sim D_{i}\backslash \rS}[\rx \in \supp(\cD)]]=(1-1/2^{i})^{n}/i $. The benchmark is attained by maximizing this quantity over integers
$i\geq2$. Write $F_n(x)=(1-2^{-x})^n/x$ for $x\geq2$.
For $n\geq4$, every integer maximizer $i_n$ satisfies
$
\lfloor\log_2 n\rfloor
\leq i_n
\leq \lceil2\log_2 n\rceil.
$
Indeed, the derivative of $\ln F_n(x)$ is
$
\frac{n\ln(2)}{2^x-1}-\frac1x = \frac{n\ln(2)-(2^{x}-1)/x}{2^x-1}.
$
Since $(2^x-1)/x$ is strictly increasing for $x\geq2$, this derivative
changes sign once, at a point in
$[\log_2 n,2\log_2 n]$. An integer maximizer is one of the two
adjacent integers. Thus $i_n=\Theta(\log n)$, so the index of a best
reference distribution grows logarithmically with the sample size.


\end{document}